\PassOptionsToPackage{table}{xcolor}
\documentclass[sigconf,screen]{acmart}

\AtBeginDocument{%
  }

\setcopyright{acmlicensed}
\copyrightyear{2025}
\acmYear{2025}
\acmDOI{10.1145/XXXXXXX.XXXXXXX} 
\acmConference[ACMMM25 '25]{ACM Multimedia 2025, Dublin, Ireland}{October 27--31, 2025}{Dublin, Ireland}
\acmISBN{978-1-4503-XXXX-X/2025/10}
\acmSubmissionID{<YourSubmissionID>}

\usepackage{amsmath}
\usepackage{mathtools}
\usepackage{amsthm}
\usepackage[utf8]{inputenc} 
\usepackage[T1]{fontenc}    
\usepackage{hyperref} 
\usepackage{url}            
\usepackage{booktabs}       
\usepackage{amsfonts}       
\usepackage{nicefrac}       
\usepackage{microtype}      
\usepackage{amssymb}
\usepackage{multirow}
\usepackage{enumitem}
\usepackage{fontawesome}
\usepackage{graphicx}
\usepackage{xspace}
\usepackage{enumitem}
\usepackage{graphicx}
\usepackage{amsmath}
\usepackage{amssymb}
\usepackage{booktabs}
\usepackage{algorithm}
\usepackage[table]{xcolor}
\usepackage{algorithmic}
\usepackage{multirow}

\usepackage[capitalize,noabbrev]{cleveref}
\theoremstyle{plain}
\newtheorem{theorem}{Theorem}[section]

\newtheorem{lemma}[theorem]{Lemma}

\theoremstyle{definition}
\newtheorem{definition}[theorem]{Definition}
\newtheorem{assumption}[theorem]{Assumption}
\theoremstyle{remark}

\acmSubmissionID{2279}

\begin{document}

\title{Optimizing Multi-Round Enhanced Training in Diffusion Models for Improved Preference Understanding}

\author{
  \begin{tabular}{c}
    Kun Li\textsuperscript{1*}, Jianhui Wang\textsuperscript{2*}, Yangfan He\textsuperscript{3*}, Xinyuan Song\textsuperscript{4}, 
    Ruoyu Wang\textsuperscript{5}, Hongyang He\textsuperscript{6}\\, Wenxin Zhang\textsuperscript{7}, 
    Jiaqi Chen\textsuperscript{8}, Keqin Li\textsuperscript{9}, Sida Li\textsuperscript{10}, 
    Miao Zhang\textsuperscript{5}, Tianyu Shi\textsuperscript{9\dag}, Xueqian Wang\textsuperscript{5} \\  
    \textsuperscript{1}Xiamen University \quad
    \textsuperscript{2}University of Electronic Science and Technology of China \quad \\  
    \textsuperscript{3}University of Minnesota—Twin Cities \quad
    \textsuperscript{4}Emory University \quad 
    \textsuperscript{5}Tsinghua University \quad \\ 
    \textsuperscript{6}University of Warwick \quad
    \textsuperscript{7}University of the Chinese Academy of Sciences  \\  
    \textsuperscript{8}George Washington University \quad
    \textsuperscript{9}University of Toronto \quad
    \textsuperscript{10}Peking University  
  \end{tabular}
}

\renewcommand{\shortauthors}{Trovato et al.}
\settopmatter{printacmref=false} 
\begin{abstract}
Generative AI has significantly changed industries by enabling text-driven image generation, yet challenges remain in achieving high-resolution outputs that align with fine-grained user preferences. Consequently, we need multi-round interactions to ensure the generated images meet their expectations. Previous methods focused on enhancing prompts to make the generated images fit with user needs using reward feedback, however, it hasn't considered optimization using multi-round dialogue dataset. In this research, We present a Visual Co-Adaptation (\textbf{VCA}) framework that incorporates human-in-the-loop feedback, utilizing a well-trained reward model specifically designed to closely align with human preferences. Leveraging a diverse multi-turn dialogue dataset, the framework applies multiple reward functions—such as diversity, consistency, and preference feedback—while fine-tuning the diffusion model through LoRA, effectively optimizing image generation based on user input. We also constructed multi-round dialogue datasets with prompts and image pairs that well fit user intent. Various experiments demonstrate the effectiveness of the proposed method over state-of-the-art baselines, with significant improvements in image consistency and alignment with user intent. Our approach consistently surpasses competing models in user satisfaction, particularly in multi-turn dialogue scenarios.
\end{abstract}

\maketitle
\vspace{-1em}

\def\thefootnote{*}\footnotetext{Equal contribution.}\def\thefootnote{\arabic{footnote}}
\def\thefootnote{\dag}\footnotetext{Corresponding author.}\def\thefootnote{\arabic{footnote}}

\section{Introduction}
\label{sec:intro}
Generative AI has become a significant driver of economic growth, enabling the optimization of both creative and non-creative tasks across a wide range of industries. Cutting-edge models such as DALL·E 2~\cite{ramesh2022hierarchical}, Imagen~\cite{saharia2022photorealistic}, and Stable Diffusion~\cite{rombach2022high} have shown extraordinary capabilities in generating unique, convincing, and lifelike images from textual descriptions. These models allow users to visualize ideas that were previously difficult to materialize, opening up new possibilities in content creation, design, advertising, and entertainment. However, there are still limitations in refining generative models. One of the key challenges lies in producing high-resolution images that more accurately reflect the semantics of the input text. Additionally, the complexity of current interfaces and the technical expertise required for effective prompt engineering remain significant obstacles for non-expert users. Many models still struggle with interpreting complex human instructions, leading to a gap between user expectations and the generated outputs. The opacity surrounding the impact of variable adjustments further complicates the user experience, particularly for those without a systematic understanding of the underlying mechanisms. This disconnect between technical complexity and user accessibility hinders a broader audience from fully leveraging the potential of these advanced AI tools. Therefore, there is a pressing need to develop more intuitive and user-friendly models that can open access to generative AI, enabling individuals from diverse backgrounds to engage in the creative process without requiring deep technical knowledge. 

To tackle these challenges, we propose an innovative approach that enhances the user experience by simplifying the interaction between users and generative models. Unlike traditional models that necessitate extensive knowledge of control elements, our approach employs the concept of human-in-the-loop co-adaptation~\cite{reddy2022first, he2024enhancing}. By integrating continuous user feedback, our model dynamically evolves to better meet user expectations. This feedback-driven adaptation not only improves the accuracy and relevance of the generated images but also reduces the learning curve for users, making the technology more accessible and empowering. Figure~\ref{fig:overall} depicts the overall workflow of our model, demonstrating how it incorporates user feedback to iteratively refine image generation. Our key contributions are summarized below:
\vspace{-1em}
\begin{figure}[!ht]
    \centering
    \includegraphics[width=\linewidth]{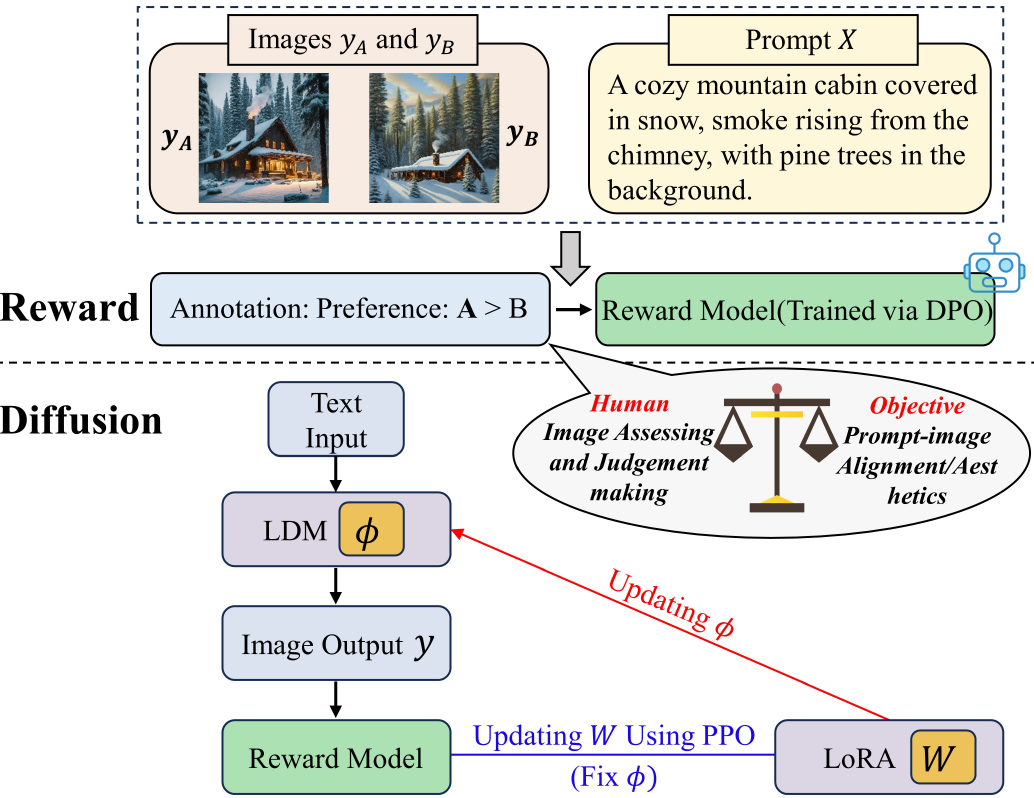} 
    \caption{The workflow demonstrates how human preferences guide text-to-image diffusion, with a DPO-trained reward model evaluating image-prompt alignment and PPO updating LoRA parameters while keeping the diffusion model fixed.}
    \label{fig:evaluation_diagram}
    \vspace{-1em}
\Description{}\end{figure}
\begin{figure*}[!ht]
    \centering
    \includegraphics[width=\textwidth]{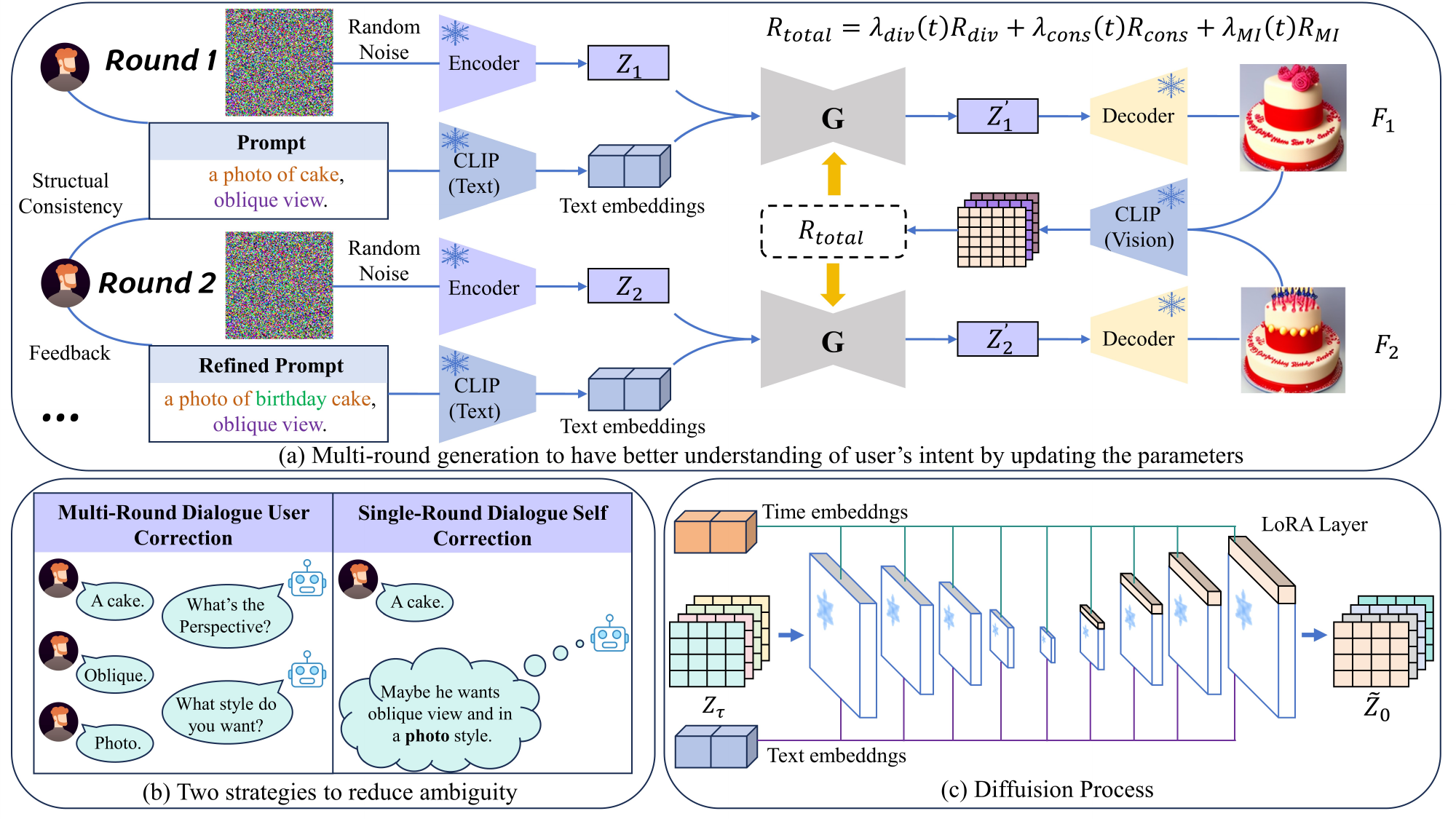}
    \caption{Overview of our multi-round dialogue generation process. (a) shows how prompts and feedback refine images over rounds. (b) compares multi-round user correction with single-round self-correction. (c) illustrates the diffusion process with LoRA layers and text embeddings. The total reward \(R_{\text{total}}\) balances diversity, consistency, and mutual information across rounds.}
    \label{fig:overall}
    \vspace{-1em}
\Description{}\end{figure*}  

\begin{itemize}[left = 0em]
    \vspace{1.1mm}
    \item We designed an extensive text-to-image dataset that considers diverse multi-turn dialogue topics, offering valuable resources for frameworks that capture human preferences.
    \item We introduced the Visual Co-Adaptation (VCA) framework with human feedback in loops, which refines user prompts using a pre-trained language model enhanced by Reinforcement Learning (RL) for the diffusion process to align image outputs more closely with user preferences. This leads to images that genuinely reflect individual styles and intentions, and maintain consistency through each round of generation.
    \item We demonstrate that mutual information maximization outperforms conventional RL in aligning model outputs with user preferences. In addition, we introduce an interactive tool that allows non-experts to create personalized, high-quality images, broadening the application of text-to-image technologies in creative fields.
    \item We provide the theoretical analysis and key insights underlying our method. These results offer a rigorous foundation that supports the effectiveness and applicability of our framework design.
\end{itemize}
\vspace{-1em}
\section{Related Work}
\subsection{Text-Driven Image Editing Framework}
Image editing through textual prompts has changed how users interact with images, making editing processes more intuitive. One of the seminal works is Prompt-to-Prompt (P2P)~\cite{hertz2022prompt}. The main idea behind P2P is aligning new information from the prompt with the cross-attention mechanism during the image generation process. P2P allows modifications to be made without retraining or adjusting the entire model, by modifying attention maps in targeted areas.

Expanding upon P2P, MUSE~\cite{chang2023muse} introduced a system that allows both textual and visual feedback during the generation process. This addition made the system more adaptable, enabling it to respond more effectively to user inputs, whether given as text or through visual corrections. Building on these developments, the Dialogue Generative Models framework~\cite{huang2024dialoggen} integrated dialogue-based interactions, enabling a conversation between the user and the model to iteratively refine the generated image. This approach improves alignment with user preferences through multiple interactions.

Prompt Charm~\cite{wang2024promptcharm} refined prompt embeddings to offer more precise control over specific image areas without requiring retraining of the entire model. More recently, ImageReward~\cite{xu2024imagereward} introduced a technique that utilizes human feedback to improve reward models, applying iterative adjustments to image outputs to align them more closely with user preferences and emphasize stronger text-to-image coherence~\cite{liang2023rich,lee2023aligning}.

\subsection{Text-to-Image Model Alignment with Human Preferences}
Following approaches like ImageReward, reinforcement learning from human feedback (RLHF) has become a key method for aligning text-to-image generation models with user preferences. RLHF refines models based on user feedback, as shown in works such as Direct Preference Optimization (DPO)~\cite{rafailov2024directpreferenceoptimizationlanguage}, Proximal Policy Optimization (PPO)~\cite{schulman2017proximalpolicyoptimizationalgorithms}, and Reinforcement Learning with Augmented Inference Feedback (RLAIF)~\cite{lee2023rlaifscalingreinforcementlearning}. These methods convert human feedback into reward signals, allowing the model to iteratively update its parameters and produce images that align with human preferences.

Achieving efficient adaptation, we adopt a LoRA-based framework, where the pre-trained weight matrices are updated via low-rank adaptation. This approach allows us to insert lightweight adapter modules into the network's attention layers for fine-tuning without modifying the full set of model parameters~\cite{xin2024v,xin2024vmt,xin2024mmap,xin2024parameter}. Recent work on enhancing intent understanding for ambiguous prompts~\cite{he2024enhancing} introduced a human-machine co-adaptation strategy that uses mutual information between user prompts and generated images, further improving alignment with user preferences in multi-round dialogue scenarios. For the models themselves, fine-tuning techniques like LoRA (Low Rank Adaptation)~\cite{hu2021loralowrankadaptationlarge} have received attention for their efficiency in updating large pre-trained models. LoRA updates a model in a low-rank subspace, preserving the original weights, which allows parameter changes to be driven by user feedback without large-scale retraining. QLoRA~\cite{dettmers2023qloraefficientfinetuningquantized} extends this by introducing 4-bit quantization to reduce memory usage, making it possible to fine-tune large models even on limited hardware. Furthermore, LoraHub~\cite{huang2024lorahubefficientcrosstaskgeneralization} enables dynamic composition of fine-tuned models for specific tasks. By combining RLHF with LoRA in our model, we can more effectively align text-to-image generation with human intent.
\vspace{-1em}
\section{Method}
\subsection{Multi-round Diffusion Process}
\label{sec:multi_round_diffusion}
The multi-round diffusion~\cite{rombach2022highresolutionimagesynthesislatent} framework introduces Gaussian noise \( \epsilon \sim \mathcal{N}(0, I) \) at each iteration to iteratively denoise latent variables \( z_t \) based on human feedback, generating progressively refined images through time-step updates. This process integrates user feedback in the prompt refinement:
\begin{equation}
P_t = \mathcal{G}(\mathcal{F}_{\text{LLM}}(\text{prompt} + \nabla_{\text{feedback}}), P_{t-1}),
\end{equation}
where the initial prompt (\(\text{prompt}\)) is refined by the LLM (\(\mathcal{F}_{\text{LLM}}\)) with feedback adjustments (\(\nabla_{\text{feedback}}\)), then aligned with the previous context (\(P_{t-1}\)) by the LLM (\(\mathcal{G}\)) to yield a new output \( P_t \). 
The refined prompt embedding \( \psi_t(P_t) \) modifies the cross-attention map, guiding the diffusion model \( G \) to iteratively denoise latent variables \( z_{\tau-1}^{(t)} \) at each timestep \(\tau\), incorporating human feedback throughout:
\begin{equation}
z_{\tau-1}^{(t)} = \text{DM}^{(t, \tau)}\left(z_\tau^{(t)}, \psi_t(P_t), \tau\right),
\end{equation}
As this framework proceeds, Gaussian noise is applied in multiple rounds using independent distributions. Noise steps \(\tau_1\) and \(\tau_2\) differ across rounds to diversify denoising paths. The final image \(\tilde{z}_{0}^{y}\) is then obtained by:
\begin{equation}
\tilde{z}_{0}^{y} = G_0^{\tau_2}(G_0^{\tau_1}(z_{\tau_1}^{x}, c_1, \psi(P)) + \mathcal{N}(\tau_2), c_2, \psi(P)),
\end{equation}
The latent variable \( z_{\tau_1}^x \) is generated by applying Gaussian noise at step \( \tau_1 \) to the input image \( x \) during a specific dialogue round. The ground-truth latent \( z_0^y \), extracted from the target image \( y \) in the subsequent dialogue round, serves as the reference for reconstruction. The prompt embedding, denoted by \( \psi(P) \), corresponds to some dialogue round. This reconstruction process is guided by the feature-level generation loss, as illustrated in Figure \ref{fig:evaluation_diagram}.
\begin{equation}
L_{\text{multi}} = \left\|z_0^y - G_0^{\tau_2}\Big(G_0^{\tau_1}\left(z_{\tau_1}^x, c_1, \psi(P)\right) + \mathcal{N}(\tau_2), c_2, \psi(P)\Big)\right\|
\end{equation}
One can simplify this by focusing on one-step reconstruction of the last timestep \( z_{\tau_2 - 1}^{y} \):
\begin{equation}
L_{\text{multi}} = \|z_{\tau_2 - 1}^{y} - G_{\tau_2 - 1}^{\tau_2}(z_{\tau_2}^{xy}, c_2,  \psi(P))\|,
\end{equation}
which simplifies optimization by allowing the loss gradient to backpropagate across rounds, ensuring coherent reconstruction and alignment of generated results with user feedback in the initial prompt.

\noindent\textbf{Theoretical Analysis.} Theorem~\ref{thms:conditional_convergence} shows that, under suitable assumptions, the distribution of the latent variables converges in total variation norm to the target distribution \(p_{\text{target}}(z)\) as the number of rounds increases. This provides a theoretical guarantee that the user’s intended content can be accurately realized over multiple rounds of diffusion, thereby validating the effectiveness of our method for tasks requiring iterative feedback and refinement.

\begin{theorem}[Conditional Convergence of Multi-Round Diffusion Process]
\label{thms:conditional_convergence}
Given a user feedback sequence \(\{\nabla_{\text{feedback}}^{(t)}\}_{t=1}^T\) that generates prompt sequences \(\{P_t\}_{t=1}^T\) via the language model \(\mathcal{F}_{\text{LLM}}\), assume there exists an ideal prompt \(P_{\text{target}}\) such that \(\psi(P_{\text{target}})\) perfectly aligns with user intent. Define the latent variable sequence \(\{z_t\}_{t=1}^T\) of the multi-round diffusion process recursively as:
\begin{equation}
z_{t-1} 
= \text{DM}^{(t)}\bigl(z_t, \psi(P_t)\bigr) 
  + \epsilon_t, 
\quad \epsilon_t 
\sim \mathcal{N}\bigl(0, \sigma_t^2 I\bigr),
\end{equation}
where \(\text{DM}^{(t)}\) is the diffusion model at round \(t\), and \(\psi(\cdot)\) is the prompt embedding function. Under Assumptions~\ref{assumption:prompt_convergence}, \ref{assumption:diffusion_stability}, and \ref{assumption:noise_decay}, as \(T \to \infty\), the generated distribution \(p(z_T)\) converges to the target distribution \(p_{\text{target}}(z)\) in total variation norm:
\begin{equation}
\lim_{T \to \infty}
\bigl\|p(z_T) - p_{\text{target}}(z)\bigr\|_{\text{TV}}
= 0.
\end{equation}
\end{theorem}

In this part, we connect the multi-round diffusion framework with our preference-guided approach. By incorporating user feedback at each round, the prompt embedding function \(\psi(\cdot)\) is updated in a manner that reduces misalignment between user intentions and generated images. The iterative noise injection and denoising steps further allow the model to explore a variety of latent states, guided by user corrections. A more detailed mathematical background, including intermediate lemmas and additional proofs, is provided in Supplementary Material~\ref{appendix:2}.

\subsection{Reward-Based Optimization in Multi-round Diffusion Process}
\label{sec:reward_based_optimization}
To guide the multi-round diffusion process more effectively, we reformulate the existing loss constraints into reward functions, allowing for a preference-driven optimization that balances diversity, consistency and mutual information.

\noindent\textbf{Diversity Reward.}
During early rounds, we stimulate diverse outputs by maximizing the diversity reward:
\begin{equation}
R_{\text{div}} = \frac{1}{N(N-1)} \sum_{i \neq j} \Bigl(1 - \frac{f(z_i) \cdot f(z_j)}{\|f(z_i)\| \|f(z_j)\|}\Bigr),
\label{eq:diversity_reward}
\end{equation}
where \( z_i \) and \( z_j \) are distinct samples from multiple rounds of prompt-text pair data in a single dialogue of the training set, and \( f(z) \) represents latent features extracted from the final layer of the UNet.

\noindent\textbf{Consistency Reward.} As dialogue rounds progress in training, we introduce a consistency reward to ensure the model's outputs between different rounds maintain high consistency. This is achieved by maximizing the cosine similarity between consecutive dialogue outputs:
\begin{equation}
R_{\text{cons}} = \sum_{t=1}^{T-1} \frac{f(z_t) \cdot f(z_{t+1})}{\|f(z_t)\| \|f(z_{t+1})\|},
\end{equation}
where \( R_{\text{cons}} \) rewards alignment and stability across multiple dialogue rounds by minimizing discrepancies between consecutive frames.

\noindent\textbf{Mutual Information Reward.} The Mutual Information Reward is computed using a custom-trained reward model derived from Qwen-VL~\cite{bai2023qwenvlversatilevisionlanguagemodel} (with the final linear layer of the Qwen-VL model removed, it calculates the logits' mean as the reward and fine-tunes using QLoRA). The model is trained using a prompt paired with two contrasting images, each labeled with 0 or 1 to indicate poor or good alignment with human intent, and optimized through DPO.
\begin{equation}
R_{\text{MI}} = I(X; Y)
\label{eq:mi_reward}
\end{equation}
The mutual information reward \( I(X; Y) \) (In our later paper, we refer to this metric as the "preference score.") is optimized to fine-tune the model's outputs, aligning the generated image with user preferences.

\noindent\textbf{Total Reward.}
The overall reward is a weighted combination of several components:
\begin{equation}
R_{\text{total}} = \lambda_{\text{div}}(t) R_{\text{div}} + \lambda_{\text{cons}}(t) R_{\text{cons}} + \lambda_{\text{MI}}(t) R_{\text{MI}}.
\label{eq : reward}
\end{equation}
Here, \( t \) is the dialogue round index. Initially, the model emphasizes image diversity with \(\lambda_{\text{div}}(t) = \exp(-\alpha t)\), which starts close to 1. The weights for consistency \(\lambda_{\text{cons}}(t) = 1 - \exp(-\beta t)\) and mutual information \(\lambda_{\text{MI}}(t) = \tfrac{1}{2}\exp(-\gamma t)\) increase over time, shifting attention from diversity to consistency and mutual information to better reflect user intent across successive rounds. In our experiments, \(\alpha = 0.15\), \(\beta = 0.1\), and \(\gamma = 0.075\). Figure~\ref{fig:weight_changes} shows how these weights vary.
\begin{figure}[!ht]
    \centering
    \includegraphics[width=\linewidth]{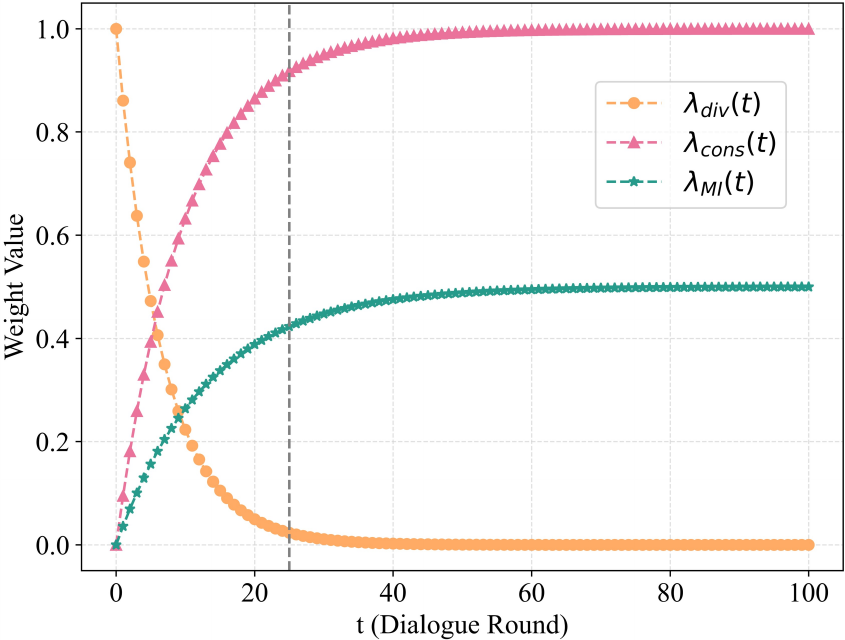}
    \caption{Weight changes for the different reward components.}
    \label{fig:weight_changes}

\Description{}\end{figure}

\subsection{Preference-Guided U-Net Optimization}
\label{sec:preference_learning}

As illustrated in Algorithm~\ref{alg:algorithm}, which outlines the preference learning process, we enhance the U-Net's image generation capabilities by embedding preference learning directly into its attention mechanisms, allowing for dynamic adjustment of the query, key, value, and output layer parameters based on rewards during the multi-round diffusion training process. Utilizing the LoRA framework, we first performed a low-rank decomposition of the weight matrix \(W\), resulting in an updated weight matrix defined as:
\begin{equation}
W_{\text{new}} = W_{\text{pretrained}} + \Delta W
\end{equation}
where:
\begin{equation}
\Delta W = B \times A + \eta \cdot \nabla_W R
\end{equation}
Matrices \(B \in \mathbb{R}^{d \times r}\) and \(A \in \mathbb{R}^{r \times d}\) correspond to the low-rank decomposition, \(\eta\) is the learning rate, and \(\nabla_W R\) is the gradient of the reward function guided by user preferences. During the multi-round diffusion training, the weight matrix is updated at each time step \(t\) as:
\begin{equation}
W_t = W_{t-1} + \eta \cdot \nabla_W R(W_{t-1}),
\end{equation} 
To ensure the generated images conform to user-defined preferences, the model minimizes the binary cross-entropy loss \(L_{\text{BCE}}\):
\begin{equation}
L_{\text{BCE}} = \bigl\|z_{t-1} - \text{DM}\bigl(z_t, \psi(P), t, W_{\text{new}}\bigr)\bigr\|,
\end{equation}
where \( z_{t-1} \) is the latent representation of the target image for the subsequent round, and \(\text{DM}(\cdot)\) indicates the model’s output at time step \( t \) based on \( z_t \), the prompt embedding \(\psi(P)\), the current time \( t \), and the updated weight matrix \( W_{\text{new}} \).

Below is a bridging paragraph that connects the above method with the theoretical framework, followed by a concise explanation demonstrating why the method is effective:

\noindent\textbf{Theoretical Analysis.} Theorem~\ref{theorems:global_optimality} guarantees that the solution sequence \(\{z_t^*\}\) converges to the Pareto optimal set, indicating that the trade-off among multiple objectives is systematically balanced as \(t\) grows. In other words, the dynamic weighting ensures that no single objective (diversity, consistency, or intent alignment) dominates the optimization in the long run, while the reward components collectively guide the model toward stable and desirable outcomes. This theoretical insight justifies the effectiveness of our approach and underlines its reliability for multi-turn diffusion tasks.

\begin{theorem}[Global Optimality of Dynamic Reward Optimization]
\label{theorems:global_optimality}
Under Assumption~\ref{assumption:dynamic_weights}, the solution sequence 
\(\{z_t^*\}_{t=1}^T\) of the optimization problem
\begin{equation}
\max_{z_t} R_{\text{total}}(t)
\end{equation}
converges to the Pareto optimal set as \(T \to \infty\), thereby achieving a balance among diversity, consistency, and intent alignment.
\end{theorem}

Building on the Preference-Guided U-Net Optimization, we adopt the dynamically weighted total reward function \(R_{\text{total}}\) to ensure that each generated image simultaneously promotes diversity, maintains temporal coherence, and aligns with user-defined semantics. The combination of LoRA-based parameter updates and careful reward design enables the model to adapt its attention mechanisms in real time, thereby strengthening the consistency of features across multiple dialogue rounds. A detailed mathematical background and the full theoretical derivations are provided in Supplementary material~\ref{appendix:2}.

\begin{algorithm}[tb]
\caption{Multi-Round Diffusion with Feedback and LoRA Fine-Tuning}
\label{alg:algorithm}
\textbf{Dataset}: Input set $\mathcal{X} = \{X_1, X_2, \dots, X_n\}$ \\
\textbf{Pre-training Dataset}: Text-image pairs dataset $\mathcal{D} = \{(\text{txt}_1, \text{img}_1), \dots, (\text{txt}_n, \text{img}_n)\}$ \\
\textbf{Input}: Initial input $X_0$, Initial noise $z_T$, LoRA parameters $W_0$, Reward model $r$, dynamic weights $\lambda_{\text{div}}$, $\lambda_{\text{cons}}$, $\lambda_{\text{MI}}$ \\
\textbf{Initialization}: Number of noise scheduler steps $T$, time step range for fine-tuning $[T_1, T_2]$ \\
\textbf{Output}: Final generated image $\hat{y}$

\begin{algorithmic}[1]
    \FOR{each $X_i \in \mathcal{X}$ and $(\text{txt}_i, \text{img}_i) \in \mathcal{D}$}
        \STATE $t \leftarrow \text{rand}(T_1, T_2)$ // Pick a random time step 
        \STATE $z_T \sim \mathcal{N}(0, I)$ // Sample noise as latent
        \FOR{$\tau = T$ to $t$}
            \STATE No gradient: $z_{\tau-1} \leftarrow \text{DM}_{W_i}^{(\tau)}\{z_\tau, \psi_t(\text{txt}_i)\}$ 
        \ENDFOR
        \STATE With gradient: $z_{t-1} \leftarrow \text{DM}_{W_i}^{(t)}\{z_t, \psi_t(\text{txt}_i)\}$
        \STATE $x_0 \leftarrow z_{t-1}$ // Predict the original latent 
        \STATE $z_i \leftarrow x_0$ // From latent to image
        
        \STATE \textbf{Reward Calculation and PPO Update:}
        \STATE Calculate diversity reward $R_{\text{div}}$, consistency reward $R_{\text{cons}}$, and mutual information reward $R_{\text{MI}}$
        \STATE Combine them to form the total reward $R_{\text{total}}$
        \STATE $L_{\text{reward}} \leftarrow \lambda R_{\text{total}}$
        \STATE Update LoRA parameters using PPO: $W_{i+1} \leftarrow W_i + \nabla_W L_{\text{reward}}$ (fixing $\phi$)
        
        \STATE \textbf{Noise Loss Calculation:}
        \STATE Calculate noise prediction loss $L_{\text{noise}} \leftarrow \| x_0 - \text{img}_i \|^2$
        \STATE Update LDM weights: $\phi_{i+1} \leftarrow \phi_{i} - \nabla_\phi L_{\text{noise}}$ (injecting LoRA parameters)
    \ENDFOR
    \STATE \textbf{return} Final updated weights $\phi_n$
\end{algorithmic}
\end{algorithm}

\vspace{-1em}
\section{Experiments}
\subsection{Experiment Settings}
We fine-tuned the Stable Diffusion v2.1 model for multi-round dialogues using LoRA with a rank of 4 and an \(\alpha\) value of 4, where \(\alpha\) is the scaling factor on parameters injected into the attention layers. The training was carried out in half precision on 4 NVIDIA A100 GPUs (40GB each), using a total batch size of 64 and a learning rate of 3e-4. For the diffusion process (Figure~\ref{fig:box_plot}), we set \(T = 70\) and \([T_1, T_2] = [1, 40]\).
\begin{figure}[!ht]
    \centering
    \includegraphics[width=\linewidth]{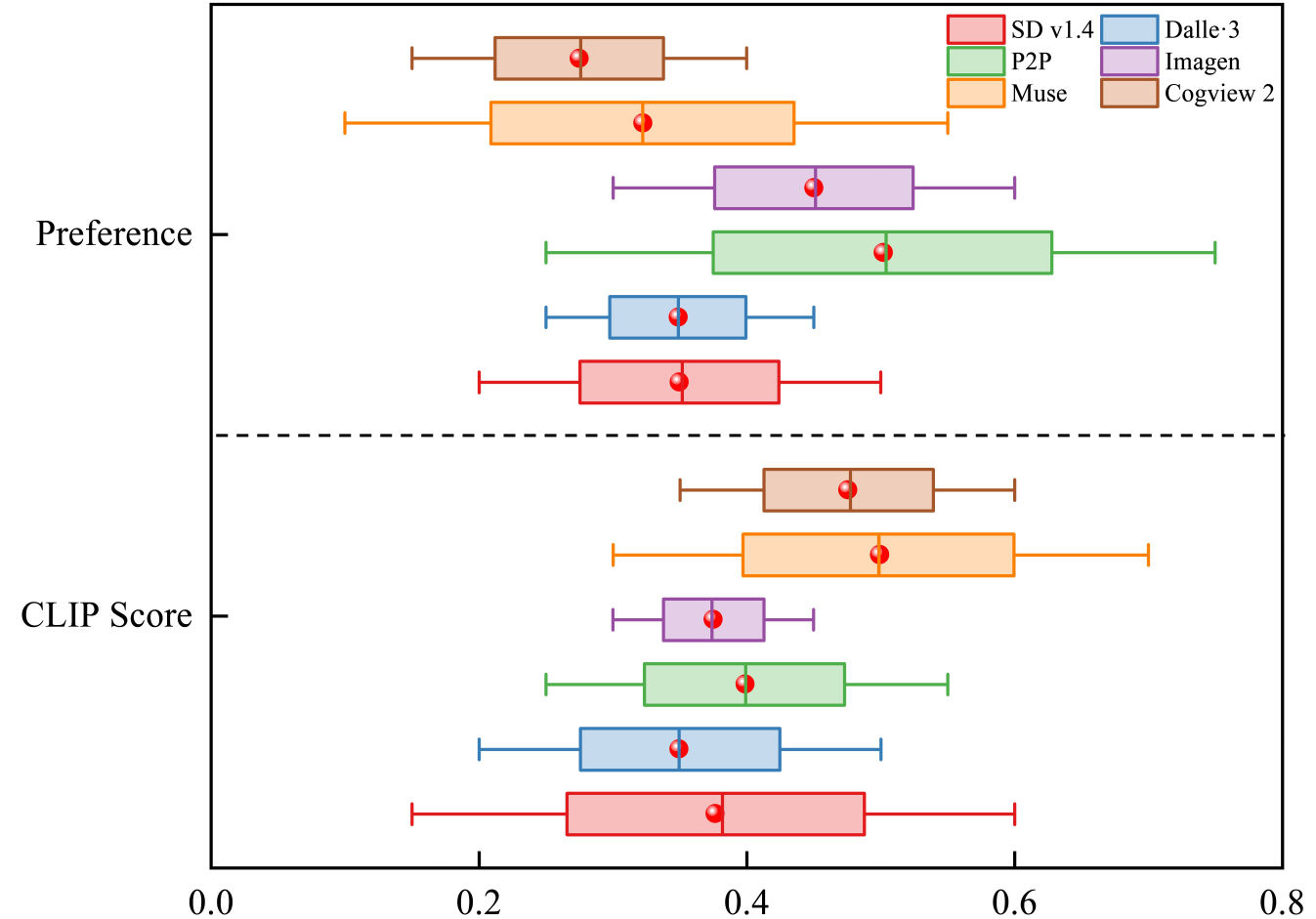} 
    \caption{Comparison of Preference and CLIP scores across different models.}
    \label{fig:box_plot}
\Description{}\end{figure} 

For the reward model, we integrated QLoRA (Quantized Low-Rank Adapter) into the transformer layers of Qwen-VL~\cite{bai2023qwenvlversatilevisionlanguagemodel}, specifically targeting its attention mechanisms and feedforward layers. QLoRA was configured with a rank of 64 and \(\alpha = 16\) for computational efficiency. The policy model was trained for one epoch with a batch size of 128 on 8 NVIDIA A100 GPUs (80GB each), using a cosine schedule to gradually reduce the learning rate.

\begin{figure}[!ht]
    \centering
    \includegraphics[width=\linewidth]{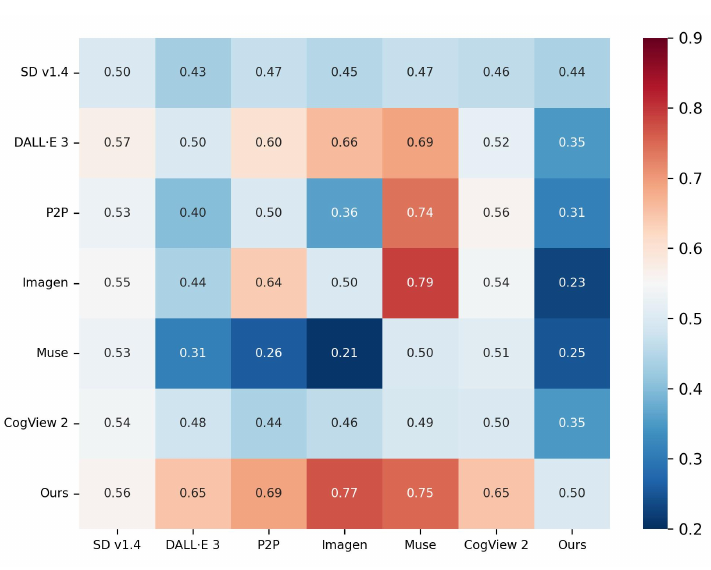} 
    \caption{Win rates between all methods.}
    \label{fig:win_rates_heatmap}
\Description{}\end{figure}
\begin{figure*}[!ht]
    \centering
    \includegraphics[width=\textwidth]{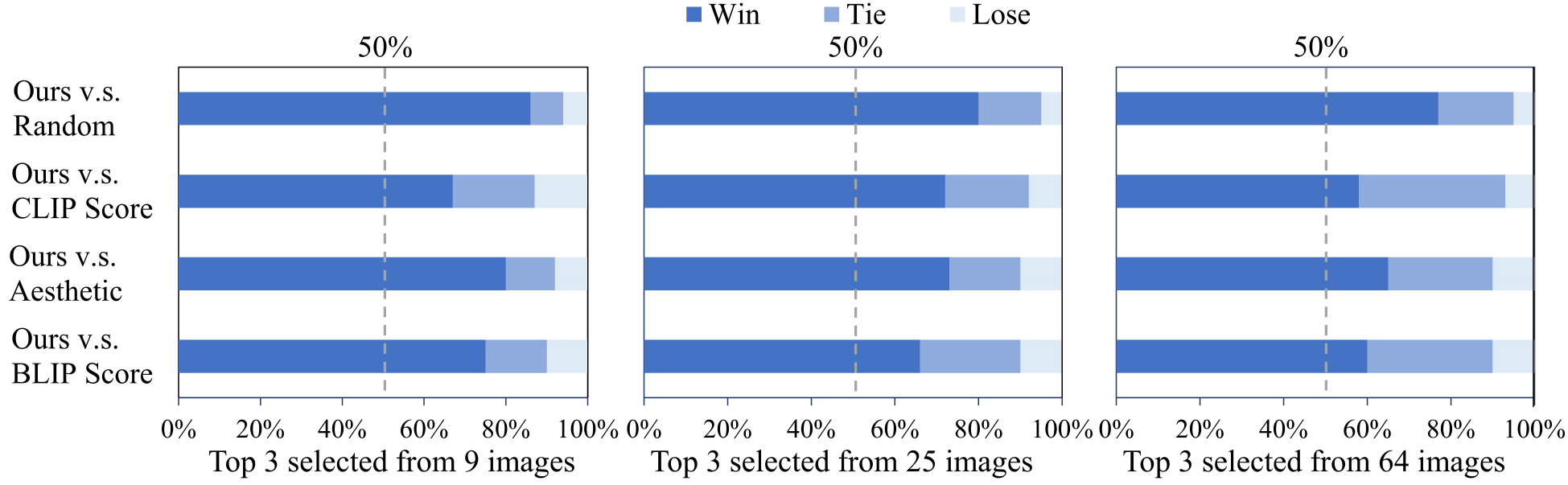}
    \caption{Win, tie, and lose rates of our model compared to Random, CLIP Score, Aesthetic, and BLIP Score across different image selection scenarios (9, 25, and 64 images).}
    \label{fig:win_rates_comparison}
\Description{}\end{figure*}
\begin{figure}[!ht]
    \centering
    \includegraphics[width=\linewidth]{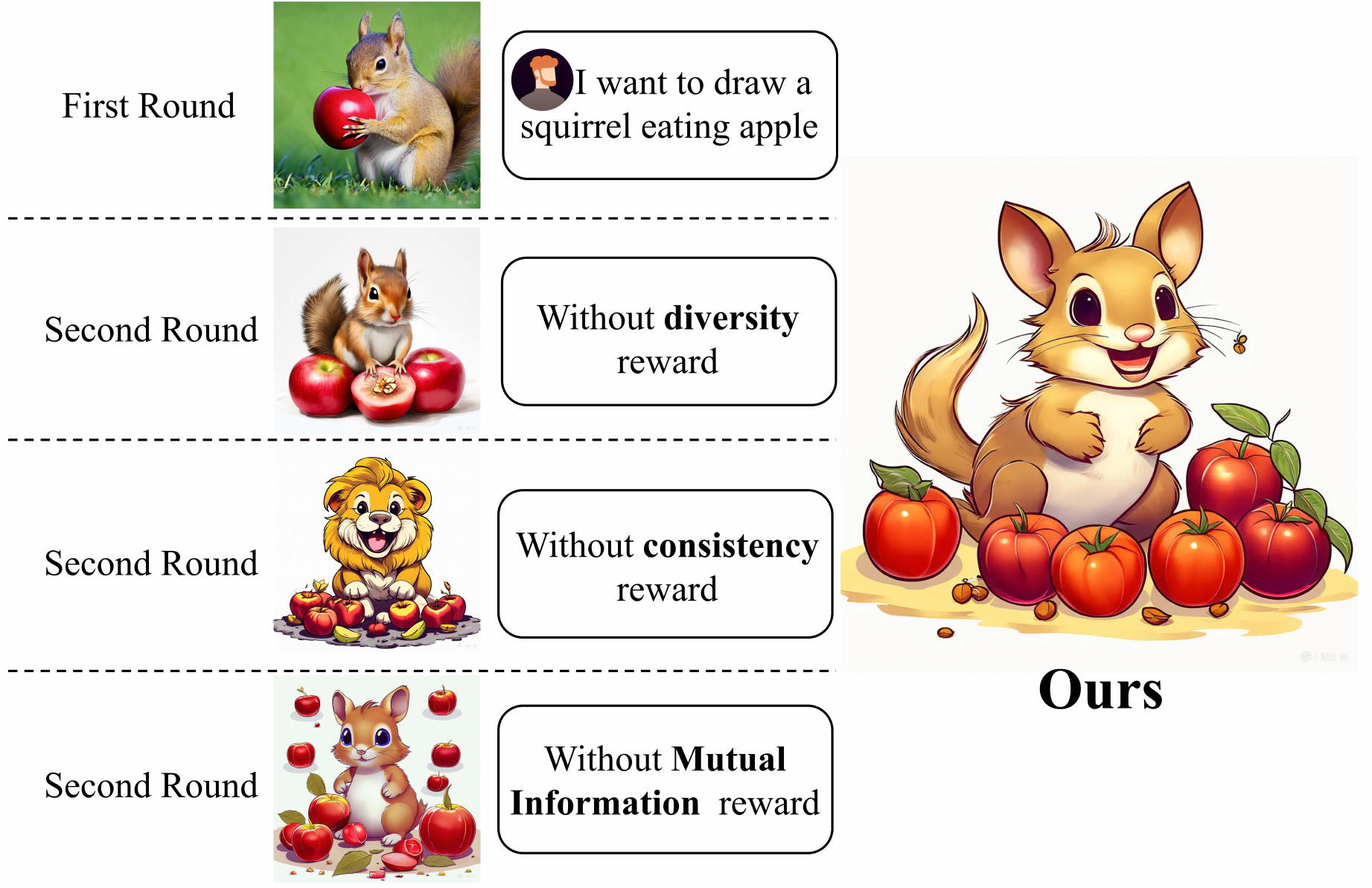} 
    \caption{Impact of removing consistency, user alignment, and diversity rewards on images generated from the prompt: "Modify this figure to a cartoon style with multiple apples and a happy squirrel."}
    \label{fig:ablation_study}
\Description{}\end{figure}
\begin{figure}[!ht]
    \centering
    \includegraphics[width=\linewidth]{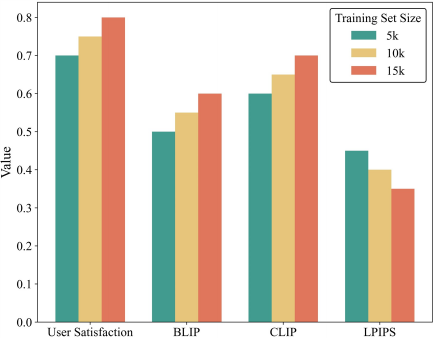}
    \caption{Comparison of user satisfaction, BLIP, CLIP, and LPIPS for different training set sizes (5k, 10k, 15k). Larger sets improve scores while lowering LPIPS.}
    \label{fig:training_set_comparison}
\Description{}\end{figure}
\subsection{Dataset}
\begin{figure}[!ht]
    \centering
    \includegraphics[width=\linewidth]{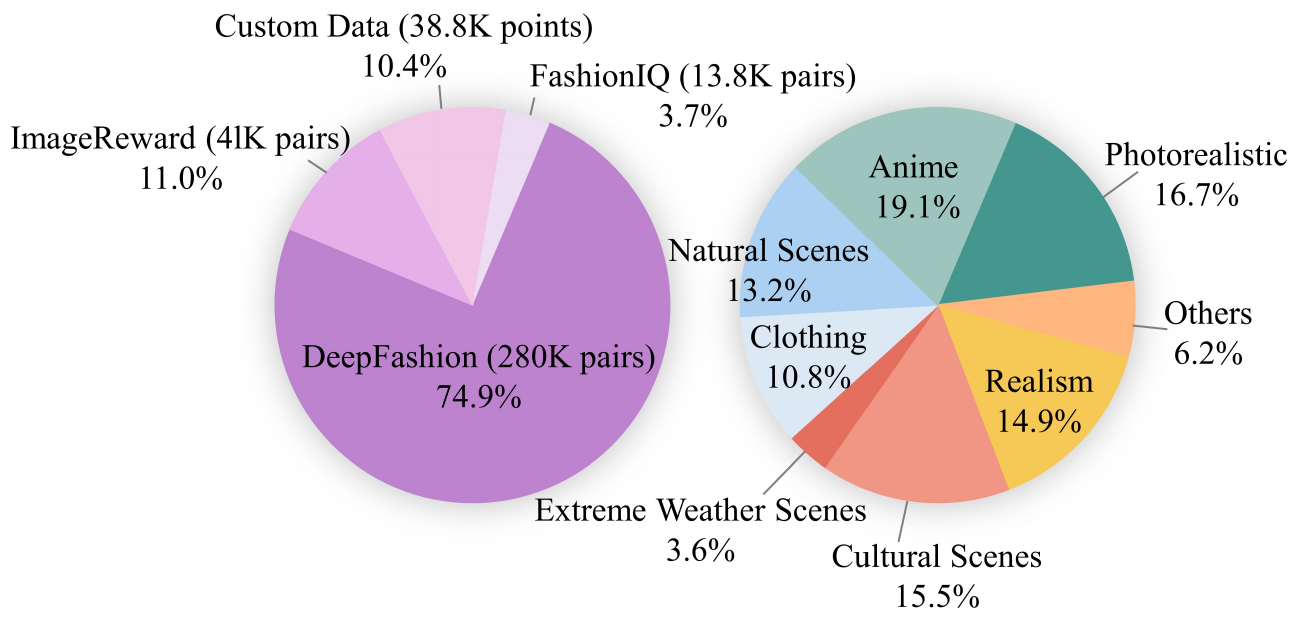} 
    \caption{Distribution of selected datasets and visual styles.}
    \label{fig:dataset_distribution}
\Description{}\end{figure} 
We selected 30\% of the ImageReward dataset, 35\% of DeepFashion~\cite{liu2016deepfashion}, and 18\% of FashionIQ~\cite{wu2020fashioniqnewdataset}, resulting in over 41K multi-round dialogue ImageReward pairs, 280K DeepFashion pairs, and 13.8K FashionIQ pairs. Additionally, we collected 38.8K custom data points (generated with QwenAI~\cite{bai2023qwentechnicalreport}, ChatGPT-4~\cite{openai2024gpt4technicalreport}, and sourced from the internet) to enhance diversity, as shown in Figure \ref{fig:dataset_distribution}. These datasets were structured into prompt-image pairs with preference labels for reward model training and text-image pairs for diffusion model training (where related text-image pairs were divided into multiple multi-round dialogues), totaling 55,832 JSON files. To ensure data quality, we excluded prompts with excessive visual style keywords and unclear images. The final dataset was split 80\% for training and 20\% for testing, with theme distributions shown in Figure~\ref{fig:dataset_distribution}.
\begin{figure*}[!ht]
    \centering
    \includegraphics[width=\textwidth]{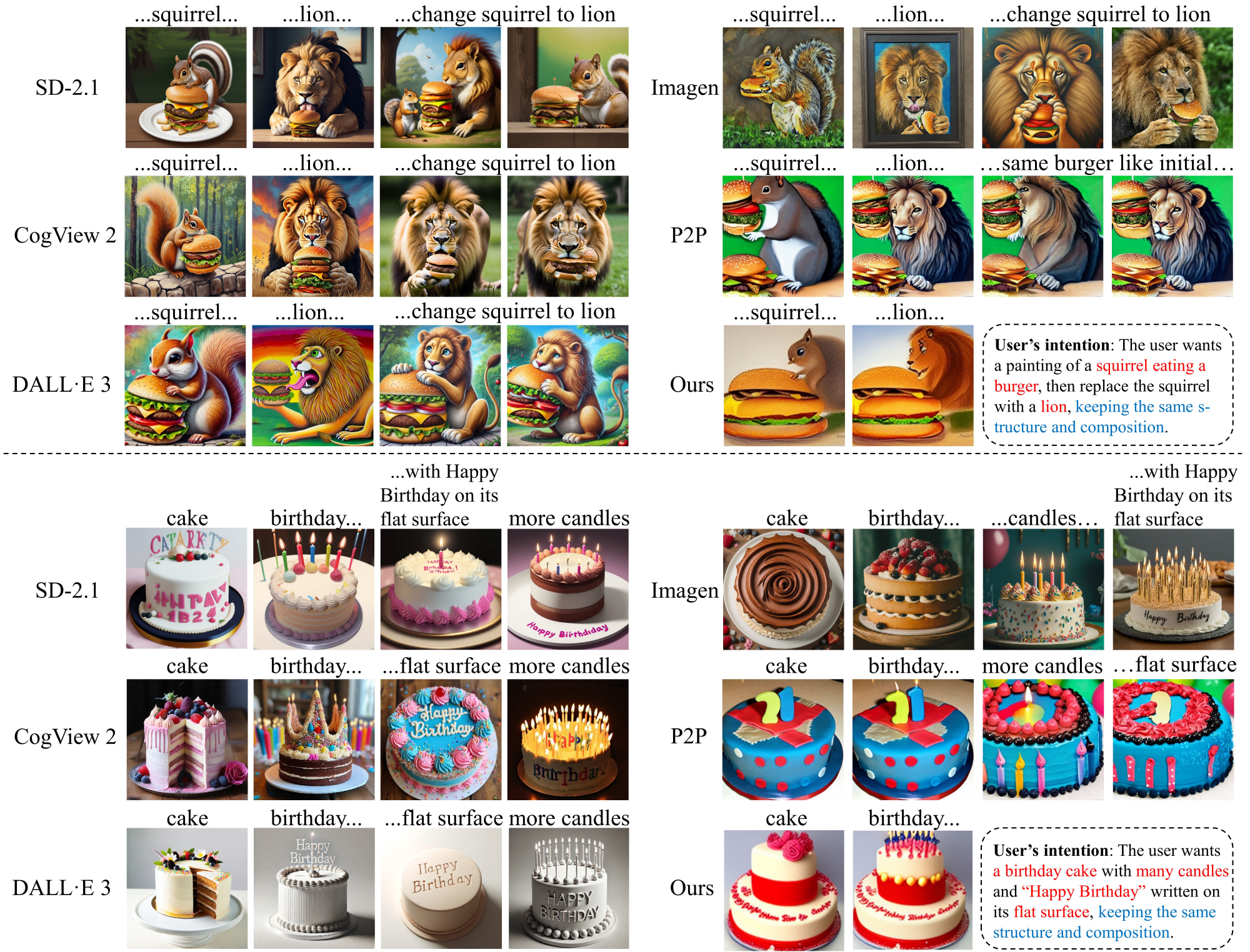}
    \caption{
The figure compares the performance of sd-2.1, Imagen, CogView2, DALL·E 3, P2P (Prompt-to-Prompt), and our model in modifying images based on user instructions.}
\label{fig:overall-compare}
\Description{}\end{figure*}
For the reward model's preference labeling, negative images were created based on the initial text-image pair (positive image). These were generated by selecting lower-scoring images in ImageReward, randomly mismatching images in FashionIQ and DeepFashion, and modifying text prompts into opposite or random ones in our custom dataset, then using the LLM to generate unrelated images. Preference labels were set to 0 for negative images and 1 for positive images. Unlike multi-turn dialogue JSON files, these were stored in a single large JSON file. We allocated 38\% of the dataset for reward model training and 17\% for diffusion model training to prevent overfitting and conserve computational resources, as the loss had already clearly converged.
\begin{table*}[!ht]
\centering
\caption{Comparison of Different Models on Real User Prompts and Testing Set. The aesthetic score is obtained from the Everypixel Aesthetic Score API, which is trained on a large user-generated dataset~\cite{wang2019youtube}, rating images from 0 to 5 based on visual appeal. For the testing set, only dialogue prompts were retained, with images omitted.}
\label{tab:model_comparison}
\small
\resizebox{\textwidth}{!}{
\begin{tabular}{cc|cccccccccc} 
\hline
\multirow{2}{*}{\textbf{Model}} & \multicolumn{1}{c|}{\textbf{Real User Prompts}} & \multicolumn{10}{c}{\textbf{Testing Set}} \\ 
\cline{2-12} 
& \textbf{Human Eval \(\uparrow\)} & \multicolumn{2}{c}{\textbf{Lpips~\cite{zhang2018unreasonable} \(\downarrow\)}} & \multicolumn{2}{c}{\textbf{Aesthetic Score}} & \multicolumn{2}{c}{\textbf{CLIP~\cite{radford2021learning} \(\uparrow\)}} & \multicolumn{2}{c}{\textbf{BLIP~\cite{li2022blip} \(\uparrow\)}} & \multicolumn{2}{c}{\textbf{Round \(\downarrow\)}} \\ 
\cline{2-12} 
& Rank (Win) & Rank & Score & Rank & Score & Rank & Score & Rank & Score & Rank & Score \\ 
\hline
SD V-1.4~\cite{rombach2022highresolutionimagesynthesislatent}                     & 8 (190) & 5 & 0.43 & 7 & 2.1 & 7 & 1.3 & 6 & 0.17 & 5 & 8.4           \\ 
Dalle-3                      & 3 (463) & 8 & 0.65 & \textbf{1} & \textbf{3.5} & 4 & 2.7 & 7 & 0.13 & 8 & 13.7           \\ 
Prompt-to-Prompt~\cite{hertz2022prompt}             & 4 (390) & 2 & 0.23 & 6 & 2.4 & 2 & 3.7 & 2 & 0.54 & 2 & 5.6            \\ 
Imagen~\cite{saharia2022photorealistic}                       & 5 (362) & 3 & 0.37 & 5 & 2.9 & 6 & 2.5 & 4 & 0.22 & 3 & 7.1            \\ 
Muse~\cite{chang2023muse}                         & 6 (340) & 4 & 0.39 & 2 & 3.3 & 3 & 3.4 & 5 & 0.21 & 7 & 12.3           \\ 
CogView 2~\cite{ding2022cogview2fasterbettertexttoimage}                     & 7 (264) & 6 & 0.47 & 3 & 3.2 & 5 & 2.6 & 7 & 0.13 & 4 & 7.2            \\ 
\hline
\textbf{Ours}                & \textbf{1 (508)} &  \textbf{1} &  \textbf{0.15} & 4 & 3.1 & \textbf{1} & \textbf{4.3} & \textbf{1} & \textbf{0.59} & \textbf{1} & \textbf{3.4}   \\ 
\textbf{Ours (Reward Coefficient: all 0.25)} & 2 (477) & 7 & 0.56 & 2 & 3.3 & 4 & 3.2 & 3 & 0.49 & 6 & 8.9   \\ 
\hline
\end{tabular}
}

\begin{minipage}{\textwidth}
\end{minipage}
\end{table*}
\begin{table*}[!ht]
\centering
\caption{Evaluation of Modality Switching for different models, with T and I representing text and image, respectively. For example, T$\rightarrow$I indicates a text-to-image request. The best result is in bold.}
\label{table:modality_switching}
\begin{tabular}{@{}lcccccccccc@{}}
\toprule
Model & \multicolumn{3}{c}{Round1} & \multicolumn{4}{c}{Round2} & \multicolumn{3}{c}{Round3} \\ \cmidrule(r){2-4} \cmidrule(l){5-8} \cmidrule(l){9-11}
 & T$\rightarrow$T & T$\rightarrow$I & I+T$\rightarrow$I & T$\rightarrow$T & T$\rightarrow$I & I+T$\rightarrow$I & I+T$\rightarrow$T & T$\rightarrow$T & T$\rightarrow$I & I+T$\rightarrow$I \\ \midrule
Qwen-VL-0-shot~\cite{bai2023qwenvlversatilevisionlanguagemodel} & 83.4 & 5.6  & 0.8  & 93.1  & 64.5  & 30.1  & 2.4  & 91.5  & 68.1  & 24.9  \\
Qwen-VL-1-shot~\cite{bai2023qwenvlversatilevisionlanguagemodel} & 85.8 & 6.1  & 0.7  & 94.2  & 67.9  & 31.0  & 2.3  & 92.4  & 70.4  & 24.3  \\
Multiroundthinking~\cite{zeng2024instilling} & 64.7  & 50.3  & 49.9  & 84.6  & 78.4  & 34.6  & 28.5  & 87.1  & 88.0  & 15.4  \\
Promptcharm~\cite{wang2024promptcharm} & 88.9  & 84.7  & 87.1  & 91.3  & 89.2  & 79.9  & 81.4  & 91.8  & \textbf{93.4}  & 71.3  \\
DialogGen\cite{huang2024dialoggen} & 77.4  & \textbf{90.4}  & 93.1  & 89.7  & 84.3  & 93.2  & 92.6  & 87.4  & 88.3  & 95.7  \\
Ours & \textbf{89.9} & 88.7  & \textbf{96.1}  & \textbf{91.6}  & \textbf{92.1}  & \textbf{95.1}  & \textbf{94.7}  & \textbf{90.9}  & 89.9  & \textbf{94.3} \\ \bottomrule
\end{tabular}
\end{table*}

\subsection{Comparison Study}
\noindent\textbf{CLIP Score vs. Our Preference Score Performance.}
Figure~\ref{fig:box_plot} illustrates a comparison of CLIP scores~\cite{hessel2022clipscorereferencefreeevaluationmetric} and user preference scores across several generative models. Our preference scores show greater distinguishability and alignment with human judgments, thanks to a wider interquartile range, while CLIP scores tend to miss these nuanced differences, as indicated by their narrower medians.\\
\noindent\textbf{Performance Comparison of Various Models.}
Table~\ref{tab:model_comparison} shows that, based on feedback from 2706 users(A blind cross-over design ensured users experienced all model variations randomly), our model achieved higher \textbf{user satisfaction} (overall assessment after testing based on response time, aesthetic score, intent reflection, and round number, with an average score of 0-5 for each sub-item) and outperformed others in visual consistency, dialogue efficiency, and intent alignment, while DALL-E 3, despite its superior aesthetics, ranked lower due to weaker intent alignment and consistency.\\
\textbf{Accuracy of Metrics Reflecting User Intentions.} Figure~\ref{fig:win_rates_comparison} compares our preference score with random selection (randomly selecting 3 images), CLIP score, aesthetic score, and BLIP~\cite{li2022blipbootstrappinglanguageimagepretraining} score (each selecting the top 3 images). The results demonstrate that our preference score consistently outperforms the others, with accuracy exceeding 50\% and remaining stable as the number of images increases, underscoring its superior ability to capture human intent.\\
\textbf{Preference Score Comparison Across Models After Few Dialogues.}
In this experiment (Figure~\ref{fig:win_rates_heatmap}), we compared the preference score win rates of various generative models (e.g., SD v1.4, DALL·E, P2P, Imagen) across 8 dialogue rounds. Our model outperformed others, with the highest win rates of 0.84 against CogView 2 and 0.78 against Muse, showcasing superior handling of complex dialogues and intent capture. Only data from the 8 rounds were included, as longer dialogues were beyond the scope of the analysis.

\noindent\textbf{Modality Switching Performance Across Models.} We further evaluate our model’s performance in Table~\ref{table:modality_switching}. It shows our model consistently excels in modality-switching tasks, particularly in text-to-image (T→I) and image-to-text (I+T→T) conversions. Based on~\cite{huang2024dialoggen}, our model demonstrates high accuracy across all rounds, outperforming competitors, especially in complex multi-modal scenarios.
\subsection{Ablation Study}
\noindent\textbf{Ablation of different rewards.} Figure~\ref{fig:ablation_study} shows that removing the diversity reward results in repetitive and less varied images, while removing the consistency reward leads to incoherent outputs that deviate from the intended style and composition. Excluding the mutual information reward causes the generated image to lose connection with the prompt, failing to accurately capture the desired theme.

\noindent\textbf{Ablation of Diffusion Model Training Size.} Figure~\ref{fig:training_set_comparison} shows that larger training sets improve user satisfaction, CLIP, and BLIP scores, though gains plateau between 10k and 15k. LPIPS~\cite{zhang2018unreasonable}  values decrease, indicating better image consistency and quality in multi-turn dialogues.

\noindent\textbf{Ablation Study on Reward Coefficient Impact.} The table~\ref{tab:model_comparison} compares our model's performance with various reward coefficients. A dynamic coefficient yields the best results in consistency, aesthetics, and CLIP/BLIP scores, while a fixed 0.25 coefficient reduces accuracy and increases variability, highlighting the superiority of a dynamic approach.

\subsection{Visualization Results}

Table~\ref{tab:model_comparison} compares our model’s performance with various reward coefficients. A dynamic coefficient yields the best results in consistency, aesthetics, and CLIP/BLIP scores, while a fixed 0.25 coefficient reduces accuracy and increases variability, illustrating the advantage of a dynamic approach.
\vspace{-1em}
\section{Conclusion}
\label{sec:conclusion}
This work addresses the challenges in user-centric text-to-image generation by integrating human feedback and advanced learning strategies. Our proposed Visual Co-Adaptation framework combines multi-turn dialogues, a tailored reward structure, and LoRA-based diffusion processes, resulting in images that better align with user intent. Experiments indicate that mutual information maximization captures preferences more effectively than standard reinforcement learning methods. In addition, our interactive tool lowers the barrier for non-experts, allowing broader participation in creative AI applications.

\bibliographystyle{ACM-Reference-Format}
\bibliography{main}


\begin{thebibliography}{38}


\ifx \showCODEN    \undefined \def \showCODEN     #1{\unskip}     \fi
\ifx \showDOI      \undefined \def \showDOI       #1{#1}\fi
\ifx \showISBNx    \undefined \def \showISBNx     #1{\unskip}     \fi
\ifx \showISBNxiii \undefined \def \showISBNxiii  #1{\unskip}     \fi
\ifx \showISSN     \undefined \def \showISSN      #1{\unskip}     \fi
\ifx \showLCCN     \undefined \def \showLCCN      #1{\unskip}     \fi
\ifx \shownote     \undefined \def \shownote      #1{#1}          \fi
\ifx \showarticletitle \undefined \def \showarticletitle #1{#1}   \fi
\ifx \showURL      \undefined \def \showURL       {\relax}        \fi
\providecommand\bibfield[2]{#2}
\providecommand\bibinfo[2]{#2}
\providecommand\natexlab[1]{#1}
\providecommand\showeprint[2][]{arXiv:#2}

\bibitem[Bai et~al\mbox{.}(2023a)]%
        {bai2023qwentechnicalreport}
\bibfield{author}{\bibinfo{person}{Jinze Bai}, \bibinfo{person}{Shuai Bai}, \bibinfo{person}{Yunfei Chu}, \bibinfo{person}{Zeyu Cui}, \bibinfo{person}{Kai Dang}, \bibinfo{person}{Xiaodong Deng}, \bibinfo{person}{Yang Fan}, \bibinfo{person}{Wenbin Ge}, \bibinfo{person}{Yu Han}, \bibinfo{person}{Fei Huang}, \bibinfo{person}{Binyuan Hui}, \bibinfo{person}{Luo Ji}, \bibinfo{person}{Mei Li}, \bibinfo{person}{Junyang Lin}, \bibinfo{person}{Runji Lin}, \bibinfo{person}{Dayiheng Liu}, \bibinfo{person}{Gao Liu}, \bibinfo{person}{Chengqiang Lu}, \bibinfo{person}{Keming Lu}, \bibinfo{person}{Jianxin Ma}, \bibinfo{person}{Rui Men}, \bibinfo{person}{Xingzhang Ren}, \bibinfo{person}{Xuancheng Ren}, \bibinfo{person}{Chuanqi Tan}, \bibinfo{person}{Sinan Tan}, \bibinfo{person}{Jianhong Tu}, \bibinfo{person}{Peng Wang}, \bibinfo{person}{Shijie Wang}, \bibinfo{person}{Wei Wang}, \bibinfo{person}{Shengguang Wu}, \bibinfo{person}{Benfeng Xu}, \bibinfo{person}{Jin Xu}, \bibinfo{person}{An Yang}, \bibinfo{person}{Hao Yang},
  \bibinfo{person}{Jian Yang}, \bibinfo{person}{Shusheng Yang}, \bibinfo{person}{Yang Yao}, \bibinfo{person}{Bowen Yu}, \bibinfo{person}{Hongyi Yuan}, \bibinfo{person}{Zheng Yuan}, \bibinfo{person}{Jianwei Zhang}, \bibinfo{person}{Xingxuan Zhang}, \bibinfo{person}{Yichang Zhang}, \bibinfo{person}{Zhenru Zhang}, \bibinfo{person}{Chang Zhou}, \bibinfo{person}{Jingren Zhou}, \bibinfo{person}{Xiaohuan Zhou}, {and} \bibinfo{person}{Tianhang Zhu}.} \bibinfo{year}{2023}\natexlab{a}.
\newblock \bibinfo{title}{Qwen Technical Report}.
\newblock
\newblock
\showeprint[arxiv]{2309.16609}~[cs.CL]
\urldef\tempurl%
\url{https://arxiv.org/abs/2309.16609}
\showURL{%
\tempurl}


\bibitem[Bai et~al\mbox{.}(2023b)]%
        {bai2023qwenvlversatilevisionlanguagemodel}
\bibfield{author}{\bibinfo{person}{Jinze Bai}, \bibinfo{person}{Shuai Bai}, \bibinfo{person}{Shusheng Yang}, \bibinfo{person}{Shijie Wang}, \bibinfo{person}{Sinan Tan}, \bibinfo{person}{Peng Wang}, \bibinfo{person}{Junyang Lin}, \bibinfo{person}{Chang Zhou}, {and} \bibinfo{person}{Jingren Zhou}.} \bibinfo{year}{2023}\natexlab{b}.
\newblock \bibinfo{title}{Qwen-VL: A Versatile Vision-Language Model for Understanding, Localization, Text Reading, and Beyond}.
\newblock
\newblock
\showeprint[arxiv]{2308.12966}~[cs.CV]
\urldef\tempurl%
\url{https://arxiv.org/abs/2308.12966}
\showURL{%
\tempurl}


\bibitem[Chang et~al\mbox{.}(2023)]%
        {chang2023muse}
\bibfield{author}{\bibinfo{person}{Huiwen Chang}, \bibinfo{person}{Han Zhang}, \bibinfo{person}{Jarred Barber}, \bibinfo{person}{AJ Maschinot}, \bibinfo{person}{Jose Lezama}, \bibinfo{person}{Lu Jiang}, \bibinfo{person}{Ming-Hsuan Yang}, \bibinfo{person}{Kevin Murphy}, \bibinfo{person}{William~T Freeman}, \bibinfo{person}{Michael Rubinstein}, {et~al\mbox{.}}} \bibinfo{year}{2023}\natexlab{}.
\newblock \bibinfo{title}{Muse: Text-to-image generation via masked generative transformers}.
\newblock
\newblock


\bibitem[Dettmers et~al\mbox{.}(2023)]%
        {dettmers2023qloraefficientfinetuningquantized}
\bibfield{author}{\bibinfo{person}{Tim Dettmers}, \bibinfo{person}{Artidoro Pagnoni}, \bibinfo{person}{Ari Holtzman}, {and} \bibinfo{person}{Luke Zettlemoyer}.} \bibinfo{year}{2023}\natexlab{}.
\newblock \bibinfo{title}{QLoRA: Efficient Finetuning of Quantized LLMs}.
\newblock
\newblock
\showeprint[arxiv]{2305.14314}~[cs.LG]
\urldef\tempurl%
\url{https://arxiv.org/abs/2305.14314}
\showURL{%
\tempurl}


\bibitem[Ding et~al\mbox{.}(2022)]%
        {ding2022cogview2fasterbettertexttoimage}
\bibfield{author}{\bibinfo{person}{Ming Ding}, \bibinfo{person}{Wendi Zheng}, \bibinfo{person}{Wenyi Hong}, {and} \bibinfo{person}{Jie Tang}.} \bibinfo{year}{2022}\natexlab{}.
\newblock \bibinfo{title}{CogView2: Faster and Better Text-to-Image Generation via Hierarchical Transformers}.
\newblock
\newblock
\showeprint[arxiv]{2204.14217}~[cs.CV]
\urldef\tempurl%
\url{https://arxiv.org/abs/2204.14217}
\showURL{%
\tempurl}


\bibitem[He et~al\mbox{.}(2024)]%
        {he2024enhancing}
\bibfield{author}{\bibinfo{person}{Yangfan He}, \bibinfo{person}{Yuxuan Bai}, {and} \bibinfo{person}{Tianyu Shi}.} \bibinfo{year}{2024}\natexlab{}.
\newblock \bibinfo{title}{Enhancing Intent Understanding for Ambiguous prompt: A Human-Machine Co-Adaption Strategy}.
\newblock
\newblock


\bibitem[Hertz et~al\mbox{.}(2022)]%
        {hertz2022prompt}
\bibfield{author}{\bibinfo{person}{Amir Hertz}, \bibinfo{person}{Ron Mokady}, \bibinfo{person}{Jay Tenenbaum}, \bibinfo{person}{Kfir Aberman}, \bibinfo{person}{Yael Pritch}, {and} \bibinfo{person}{Daniel Cohen-Or}.} \bibinfo{year}{2022}\natexlab{}.
\newblock \bibinfo{title}{Prompt-to-prompt image editing with cross attention control}.
\newblock
\newblock


\bibitem[Hessel et~al\mbox{.}(2022)]%
        {hessel2022clipscorereferencefreeevaluationmetric}
\bibfield{author}{\bibinfo{person}{Jack Hessel}, \bibinfo{person}{Ari Holtzman}, \bibinfo{person}{Maxwell Forbes}, \bibinfo{person}{Ronan~Le Bras}, {and} \bibinfo{person}{Yejin Choi}.} \bibinfo{year}{2022}\natexlab{}.
\newblock \bibinfo{title}{CLIPScore: A Reference-free Evaluation Metric for Image Captioning}.
\newblock
\newblock
\showeprint[arxiv]{2104.08718}~[cs.CV]
\urldef\tempurl%
\url{https://arxiv.org/abs/2104.08718}
\showURL{%
\tempurl}


\bibitem[Hu et~al\mbox{.}(2021)]%
        {hu2021loralowrankadaptationlarge}
\bibfield{author}{\bibinfo{person}{Edward~J. Hu}, \bibinfo{person}{Yelong Shen}, \bibinfo{person}{Phillip Wallis}, \bibinfo{person}{Zeyuan Allen-Zhu}, \bibinfo{person}{Yuanzhi Li}, \bibinfo{person}{Shean Wang}, \bibinfo{person}{Lu Wang}, {and} \bibinfo{person}{Weizhu Chen}.} \bibinfo{year}{2021}\natexlab{}.
\newblock \bibinfo{title}{LoRA: Low-Rank Adaptation of Large Language Models}.
\newblock
\newblock
\showeprint[arxiv]{2106.09685}~[cs.CL]
\urldef\tempurl%
\url{https://arxiv.org/abs/2106.09685}
\showURL{%
\tempurl}


\bibitem[Huang et~al\mbox{.}(2024a)]%
        {huang2024lorahubefficientcrosstaskgeneralization}
\bibfield{author}{\bibinfo{person}{Chengsong Huang}, \bibinfo{person}{Qian Liu}, \bibinfo{person}{Bill~Yuchen Lin}, \bibinfo{person}{Tianyu Pang}, \bibinfo{person}{Chao Du}, {and} \bibinfo{person}{Min Lin}.} \bibinfo{year}{2024}\natexlab{a}.
\newblock \bibinfo{title}{LoraHub: Efficient Cross-Task Generalization via Dynamic LoRA Composition}.
\newblock
\newblock
\showeprint[arxiv]{2307.13269}~[cs.CL]
\urldef\tempurl%
\url{https://arxiv.org/abs/2307.13269}
\showURL{%
\tempurl}


\bibitem[Huang et~al\mbox{.}(2024b)]%
        {huang2024dialoggen}
\bibfield{author}{\bibinfo{person}{Minbin Huang}, \bibinfo{person}{Yanxin Long}, \bibinfo{person}{Xinchi Deng}, \bibinfo{person}{Ruihang Chu}, \bibinfo{person}{Jiangfeng Xiong}, \bibinfo{person}{Xiaodan Liang}, \bibinfo{person}{Hong Cheng}, \bibinfo{person}{Qinglin Lu}, {and} \bibinfo{person}{Wei Liu}.} \bibinfo{year}{2024}\natexlab{b}.
\newblock \bibinfo{title}{DialogGen: Multi-modal Interactive Dialogue System for Multi-turn Text-to-Image Generation}.
\newblock
\newblock


\bibitem[Lee et~al\mbox{.}(2023b)]%
        {lee2023rlaifscalingreinforcementlearning}
\bibfield{author}{\bibinfo{person}{Harrison Lee}, \bibinfo{person}{Samrat Phatale}, \bibinfo{person}{Hassan Mansoor}, \bibinfo{person}{Thomas Mesnard}, \bibinfo{person}{Johan Ferret}, \bibinfo{person}{Kellie Lu}, \bibinfo{person}{Colton Bishop}, \bibinfo{person}{Ethan Hall}, \bibinfo{person}{Victor Carbune}, \bibinfo{person}{Abhinav Rastogi}, {and} \bibinfo{person}{Sushant Prakash}.} \bibinfo{year}{2023}\natexlab{b}.
\newblock \bibinfo{title}{RLAIF: Scaling Reinforcement Learning from Human Feedback with AI Feedback}.
\newblock
\newblock
\showeprint[arxiv]{2309.00267}~[cs.CL]
\urldef\tempurl%
\url{https://arxiv.org/abs/2309.00267}
\showURL{%
\tempurl}


\bibitem[Lee et~al\mbox{.}(2023a)]%
        {lee2023aligning}
\bibfield{author}{\bibinfo{person}{Kimin Lee}, \bibinfo{person}{Hao Liu}, \bibinfo{person}{Moonkyung Ryu}, \bibinfo{person}{Olivia Watkins}, \bibinfo{person}{Yuqing Du}, \bibinfo{person}{Craig Boutilier}, \bibinfo{person}{Pieter Abbeel}, \bibinfo{person}{Mohammad Ghavamzadeh}, {and} \bibinfo{person}{Shixiang~Shane Gu}.} \bibinfo{year}{2023}\natexlab{a}.
\newblock \bibinfo{title}{Aligning text-to-image models using human feedback}.
\newblock
\newblock


\bibitem[Li et~al\mbox{.}(2022a)]%
        {li2022blip}
\bibfield{author}{\bibinfo{person}{Junnan Li}, \bibinfo{person}{Dongxu Li}, \bibinfo{person}{Caiming Xiong}, {and} \bibinfo{person}{Steven Hoi}.} \bibinfo{year}{2022}\natexlab{a}.
\newblock \bibinfo{title}{Blip: Bootstrapping language-image pre-training for unified vision-language understanding and generation}.
\newblock , \bibinfo{numpages}{12888--12900}~pages.
\newblock


\bibitem[Li et~al\mbox{.}(2022b)]%
        {li2022blipbootstrappinglanguageimagepretraining}
\bibfield{author}{\bibinfo{person}{Junnan Li}, \bibinfo{person}{Dongxu Li}, \bibinfo{person}{Caiming Xiong}, {and} \bibinfo{person}{Steven Hoi}.} \bibinfo{year}{2022}\natexlab{b}.
\newblock \bibinfo{title}{BLIP: Bootstrapping Language-Image Pre-training for Unified Vision-Language Understanding and Generation}.
\newblock
\newblock
\showeprint[arxiv]{2201.12086}~[cs.CV]
\urldef\tempurl%
\url{https://arxiv.org/abs/2201.12086}
\showURL{%
\tempurl}


\bibitem[Liang et~al\mbox{.}(2023)]%
        {liang2023rich}
\bibfield{author}{\bibinfo{person}{Youwei Liang}, \bibinfo{person}{Junfeng He}, \bibinfo{person}{Gang Li}, \bibinfo{person}{Peizhao Li}, \bibinfo{person}{Arseniy Klimovskiy}, \bibinfo{person}{Nicholas Carolan}, \bibinfo{person}{Jiao Sun}, \bibinfo{person}{Jordi Pont-Tuset}, \bibinfo{person}{Sarah Young}, \bibinfo{person}{Feng Yang}, {et~al\mbox{.}}} \bibinfo{year}{2023}\natexlab{}.
\newblock \bibinfo{title}{Rich Human Feedback for Text-to-Image Generation}.
\newblock
\newblock


\bibitem[Liu et~al\mbox{.}(2016)]%
        {liu2016deepfashion}
\bibfield{author}{\bibinfo{person}{Ziwei Liu}, \bibinfo{person}{Ping Luo}, \bibinfo{person}{Shi Qiu}, \bibinfo{person}{Xiaogang Wang}, {and} \bibinfo{person}{Xiaoou Tang}.} \bibinfo{year}{2016}\natexlab{}.
\newblock \bibinfo{title}{DeepFashion: Powering Robust Clothes Recognition and Retrieval with Rich Annotations}.
\newblock
\newblock


\bibitem[Miettinen(1999)]%
        {miettinen1999nonlinear}
\bibfield{author}{\bibinfo{person}{Kaisa Miettinen}.} \bibinfo{year}{1999}\natexlab{}.
\newblock \bibinfo{title}{Nonlinear Multiobjective Optimization}.
\newblock
\newblock
\showISBNx{978-0-7923-8278-2}


\bibitem[OpenAI et~al\mbox{.}(2024)]%
        {openai2024gpt4technicalreport}
\bibfield{author}{\bibinfo{person}{OpenAI}, \bibinfo{person}{Josh Achiam}, \bibinfo{person}{Steven Adler}, {and} \bibinfo{person}{Sandhini~Agarwal et al.}} \bibinfo{year}{2024}\natexlab{}.
\newblock \bibinfo{title}{GPT-4 Technical Report}.
\newblock
\newblock
\showeprint[arxiv]{2303.08774}~[cs.CL]
\urldef\tempurl%
\url{https://arxiv.org/abs/2303.08774}
\showURL{%
\tempurl}


\bibitem[Radford et~al\mbox{.}(2021)]%
        {radford2021learning}
\bibfield{author}{\bibinfo{person}{Alec Radford}, \bibinfo{person}{Jong~Wook Kim}, \bibinfo{person}{Chris Hallacy}, \bibinfo{person}{Aditya Ramesh}, \bibinfo{person}{Gabriel Goh}, \bibinfo{person}{Sandhini Agarwal}, \bibinfo{person}{Girish Sastry}, \bibinfo{person}{Amanda Askell}, \bibinfo{person}{Pamela Mishkin}, \bibinfo{person}{Jack Clark}, {et~al\mbox{.}}} \bibinfo{year}{2021}\natexlab{}.
\newblock \bibinfo{title}{Learning transferable visual models from natural language supervision}.
\newblock , \bibinfo{numpages}{8748--8763}~pages.
\newblock


\bibitem[Rafailov et~al\mbox{.}(2024)]%
        {rafailov2024directpreferenceoptimizationlanguage}
\bibfield{author}{\bibinfo{person}{Rafael Rafailov}, \bibinfo{person}{Archit Sharma}, \bibinfo{person}{Eric Mitchell}, \bibinfo{person}{Stefano Ermon}, \bibinfo{person}{Christopher~D. Manning}, {and} \bibinfo{person}{Chelsea Finn}.} \bibinfo{year}{2024}\natexlab{}.
\newblock \bibinfo{title}{Direct Preference Optimization: Your Language Model is Secretly a Reward Model}.
\newblock
\newblock
\showeprint[arxiv]{2305.18290}~[cs.LG]
\urldef\tempurl%
\url{https://arxiv.org/abs/2305.18290}
\showURL{%
\tempurl}


\bibitem[Ramesh et~al\mbox{.}(2022)]%
        {ramesh2022hierarchical}
\bibfield{author}{\bibinfo{person}{Aditya Ramesh}, \bibinfo{person}{Prafulla Dhariwal}, \bibinfo{person}{Alex Nichol}, \bibinfo{person}{Casey Chu}, {and} \bibinfo{person}{Mark Chen}.} \bibinfo{year}{2022}\natexlab{}.
\newblock \bibinfo{title}{Hierarchical text-conditional image generation with clip latents}.
\newblock , \bibinfo{numpages}{3}~pages.
\newblock


\bibitem[Reddy et~al\mbox{.}(2022)]%
        {reddy2022first}
\bibfield{author}{\bibinfo{person}{Siddharth Reddy}, \bibinfo{person}{Sergey Levine}, {and} \bibinfo{person}{Anca Dragan}.} \bibinfo{year}{2022}\natexlab{}.
\newblock \bibinfo{title}{First contact: Unsupervised human-machine co-adaptation via mutual information maximization}.
\newblock , \bibinfo{numpages}{31542--31556}~pages.
\newblock


\bibitem[Rombach et~al\mbox{.}(2022a)]%
        {rombach2022high}
\bibfield{author}{\bibinfo{person}{Robin Rombach}, \bibinfo{person}{Andreas Blattmann}, \bibinfo{person}{Dominik Lorenz}, \bibinfo{person}{Patrick Esser}, {and} \bibinfo{person}{Bj{\"o}rn Ommer}.} \bibinfo{year}{2022}\natexlab{a}.
\newblock \bibinfo{title}{High-resolution image synthesis with latent diffusion models}.
\newblock , \bibinfo{numpages}{10684--10695}~pages.
\newblock


\bibitem[Rombach et~al\mbox{.}(2022b)]%
        {rombach2022highresolutionimagesynthesislatent}
\bibfield{author}{\bibinfo{person}{Robin Rombach}, \bibinfo{person}{Andreas Blattmann}, \bibinfo{person}{Dominik Lorenz}, \bibinfo{person}{Patrick Esser}, {and} \bibinfo{person}{Björn Ommer}.} \bibinfo{year}{2022}\natexlab{b}.
\newblock \bibinfo{title}{High-Resolution Image Synthesis with Latent Diffusion Models}.
\newblock
\newblock
\showeprint[arxiv]{2112.10752}~[cs.CV]
\urldef\tempurl%
\url{https://arxiv.org/abs/2112.10752}
\showURL{%
\tempurl}


\bibitem[Saharia et~al\mbox{.}(2022)]%
        {saharia2022photorealistic}
\bibfield{author}{\bibinfo{person}{Chitwan Saharia}, \bibinfo{person}{William Chan}, \bibinfo{person}{Saurabh Saxena}, \bibinfo{person}{Lala Li}, \bibinfo{person}{Jay Whang}, \bibinfo{person}{Emily~L Denton}, \bibinfo{person}{Kamyar Ghasemipour}, \bibinfo{person}{Raphael Gontijo~Lopes}, \bibinfo{person}{Burcu Karagol~Ayan}, \bibinfo{person}{Tim Salimans}, {et~al\mbox{.}}} \bibinfo{year}{2022}\natexlab{}.
\newblock \bibinfo{title}{Photorealistic text-to-image diffusion models with deep language understanding}.
\newblock , \bibinfo{numpages}{36479--36494}~pages.
\newblock


\bibitem[Scheffé(1947)]%
        {scheffe1947useful}
\bibfield{author}{\bibinfo{person}{Henry Scheffé}.} \bibinfo{year}{1947}\natexlab{}.
\newblock \showarticletitle{A Useful Convergence Theorem for Probability Distributions}.
\newblock \bibinfo{journal}{\emph{Annals of Mathematical Statistics}} \bibinfo{volume}{18}, \bibinfo{number}{3} (\bibinfo{year}{1947}), \bibinfo{pages}{434--438}.
\newblock
\urldef\tempurl%
\url{https://doi.org/10.1214/aoms/1177730386}
\showDOI{\tempurl}


\bibitem[Schulman et~al\mbox{.}(2017)]%
        {schulman2017proximalpolicyoptimizationalgorithms}
\bibfield{author}{\bibinfo{person}{John Schulman}, \bibinfo{person}{Filip Wolski}, \bibinfo{person}{Prafulla Dhariwal}, \bibinfo{person}{Alec Radford}, {and} \bibinfo{person}{Oleg Klimov}.} \bibinfo{year}{2017}\natexlab{}.
\newblock \bibinfo{title}{Proximal Policy Optimization Algorithms}.
\newblock
\newblock
\showeprint[arxiv]{1707.06347}~[cs.LG]
\urldef\tempurl%
\url{https://arxiv.org/abs/1707.06347}
\showURL{%
\tempurl}


\bibitem[Wang et~al\mbox{.}(2019)]%
        {wang2019youtube}
\bibfield{author}{\bibinfo{person}{Yilin Wang}, \bibinfo{person}{Sasi Inguva}, {and} \bibinfo{person}{Balu Adsumilli}.} \bibinfo{year}{2019}\natexlab{}.
\newblock \bibinfo{title}{YouTube UGC dataset for video compression research}.
\newblock , \bibinfo{numpages}{5}~pages.
\newblock


\bibitem[Wang et~al\mbox{.}(2024)]%
        {wang2024promptcharm}
\bibfield{author}{\bibinfo{person}{Zhijie Wang}, \bibinfo{person}{Yuheng Huang}, \bibinfo{person}{Da Song}, \bibinfo{person}{Lei Ma}, {and} \bibinfo{person}{Tianyi Zhang}.} \bibinfo{year}{2024}\natexlab{}.
\newblock \bibinfo{title}{PromptCharm: Text-to-Image Generation through Multi-modal Prompting and Refinement}.
\newblock , \bibinfo{numpages}{21}~pages.
\newblock


\bibitem[Wu et~al\mbox{.}(2020)]%
        {wu2020fashioniqnewdataset}
\bibfield{author}{\bibinfo{person}{Hui Wu}, \bibinfo{person}{Yupeng Gao}, \bibinfo{person}{Xiaoxiao Guo}, \bibinfo{person}{Ziad Al-Halah}, \bibinfo{person}{Steven Rennie}, \bibinfo{person}{Kristen Grauman}, {and} \bibinfo{person}{Rogerio Feris}.} \bibinfo{year}{2020}\natexlab{}.
\newblock \bibinfo{title}{Fashion IQ: A New Dataset Towards Retrieving Images by Natural Language Feedback}.
\newblock
\newblock
\showeprint[arxiv]{1905.12794}~[cs.CV]
\urldef\tempurl%
\url{https://arxiv.org/abs/1905.12794}
\showURL{%
\tempurl}


\bibitem[Xin et~al\mbox{.}(2024a)]%
        {xin2024vmt}
\bibfield{author}{\bibinfo{person}{Yi Xin}, \bibinfo{person}{Junlong Du}, \bibinfo{person}{Qiang Wang}, \bibinfo{person}{Zhiwen Lin}, {and} \bibinfo{person}{Ke Yan}.} \bibinfo{year}{2024}\natexlab{a}.
\newblock \showarticletitle{VMT-Adapter: Parameter-Efficient Transfer Learning for Multi-Task Dense Scene Understanding}. In \bibinfo{booktitle}{\emph{Proceedings of the AAAI Conference on Artificial Intelligence}}, Vol.~\bibinfo{volume}{38}. \bibinfo{pages}{16085--16093}.
\newblock


\bibitem[Xin et~al\mbox{.}(2024b)]%
        {xin2024mmap}
\bibfield{author}{\bibinfo{person}{Yi Xin}, \bibinfo{person}{Junlong Du}, \bibinfo{person}{Qiang Wang}, \bibinfo{person}{Ke Yan}, {and} \bibinfo{person}{Shouhong Ding}.} \bibinfo{year}{2024}\natexlab{b}.
\newblock \showarticletitle{MmAP: Multi-modal Alignment Prompt for Cross-domain Multi-task Learning}. In \bibinfo{booktitle}{\emph{Proceedings of the AAAI Conference on Artificial Intelligence}}, Vol.~\bibinfo{volume}{38}. \bibinfo{pages}{16076--16084}.
\newblock


\bibitem[Xin et~al\mbox{.}(2024c)]%
        {xin2024v}
\bibfield{author}{\bibinfo{person}{Yi Xin}, \bibinfo{person}{Siqi Luo}, \bibinfo{person}{Xuyang Liu}, \bibinfo{person}{Haodi Zhou}, \bibinfo{person}{Xinyu Cheng}, \bibinfo{person}{Christina~E Lee}, \bibinfo{person}{Junlong Du}, \bibinfo{person}{Haozhe Wang}, \bibinfo{person}{MingCai Chen}, \bibinfo{person}{Ting Liu}, {et~al\mbox{.}}} \bibinfo{year}{2024}\natexlab{c}.
\newblock \showarticletitle{V-petl bench: A unified visual parameter-efficient transfer learning benchmark}.
\newblock \bibinfo{journal}{\emph{Advances in Neural Information Processing Systems}}  \bibinfo{volume}{37} (\bibinfo{year}{2024}), \bibinfo{pages}{80522--80535}.
\newblock


\bibitem[Xin et~al\mbox{.}(2024d)]%
        {xin2024parameter}
\bibfield{author}{\bibinfo{person}{Yi Xin}, \bibinfo{person}{Siqi Luo}, \bibinfo{person}{Haodi Zhou}, \bibinfo{person}{Junlong Du}, \bibinfo{person}{Xiaohong Liu}, \bibinfo{person}{Yue Fan}, \bibinfo{person}{Qing Li}, {and} \bibinfo{person}{Yuntao Du}.} \bibinfo{year}{2024}\natexlab{d}.
\newblock \showarticletitle{Parameter-efficient fine-tuning for pre-trained vision models: A survey}.
\newblock \bibinfo{journal}{\emph{arXiv preprint arXiv:2402.02242}} (\bibinfo{year}{2024}).
\newblock


\bibitem[Xu et~al\mbox{.}(2024)]%
        {xu2024imagereward}
\bibfield{author}{\bibinfo{person}{Jiazheng Xu}, \bibinfo{person}{Xiao Liu}, \bibinfo{person}{Yuchen Wu}, \bibinfo{person}{Yuxuan Tong}, \bibinfo{person}{Qinkai Li}, \bibinfo{person}{Ming Ding}, \bibinfo{person}{Jie Tang}, {and} \bibinfo{person}{Yuxiao Dong}.} \bibinfo{year}{2024}\natexlab{}.
\newblock \bibinfo{title}{Imagereward: Learning and evaluating human preferences for text-to-image generation}.
\newblock
\newblock


\bibitem[Zeng et~al\mbox{.}(2024)]%
        {zeng2024instilling}
\bibfield{author}{\bibinfo{person}{Lidong Zeng}, \bibinfo{person}{Zhedong Zheng}, \bibinfo{person}{Yinwei Wei}, {and} \bibinfo{person}{Tat-seng Chua}.} \bibinfo{year}{2024}\natexlab{}.
\newblock \bibinfo{title}{Instilling Multi-round Thinking to Text-guided Image Generation}.
\newblock
\newblock


\bibitem[Zhang et~al\mbox{.}(2018)]%
        {zhang2018unreasonable}
\bibfield{author}{\bibinfo{person}{Richard Zhang}, \bibinfo{person}{Phillip Isola}, \bibinfo{person}{Alexei~A Efros}, \bibinfo{person}{Eli Shechtman}, {and} \bibinfo{person}{Oliver Wang}.} \bibinfo{year}{2018}\natexlab{}.
\newblock \bibinfo{title}{The unreasonable effectiveness of deep features as a perceptual metric}.
\newblock , \bibinfo{numpages}{586--595}~pages.
\newblock


\end{thebibliography}
\setcounter{section}{0}
\setcounter{table}{0}
\setcounter{figure}{0}
\setcounter{equation}{0}
\clearpage
\appendix

\section{Theoretical Analysis}

\subsection{Conditional Convergence of Multi-Round Diffusion Process}\label{appendix:1}
\begin{assumption}[Prompt Convergence Condition]
\label{assumption:prompt_convergence}
For some \(\alpha \in (0,1)\), the prompt sequence satisfies
\begin{equation}
\bigl\|\psi(P_t) - \psi(P_{\text{target}})\bigr\|_2 
\leq \alpha^t 
         \bigl\|\psi(P_0) - \psi(P_{\text{target}})\bigr\|_2.
\end{equation}
\end{assumption}

\begin{assumption}[Diffusion Model Stability]
\label{assumption:diffusion_stability}
There exists \(\beta < 1\) such that for any \(z, z'\) and embedding \(\psi\),
\begin{equation}
\bigl\|\text{DM}^{(t)}(z,\psi) - \text{DM}^{(t)}(z',\psi)\bigr\|_2 
\leq 
\beta \bigl\|z - z'\bigr\|_2.
\end{equation}
\end{assumption}

\begin{assumption}[Noise Decay Condition]
\label{assumption:noise_decay}
The noise term \(\sigma_t\) satisfies 
\begin{equation}
\sigma_t = o\Bigl(\tfrac{1}{t}\Bigr).
\end{equation}
\end{assumption}

\begin{theorem}[Conditional Convergence of Multi-Round Diffusion Process. Theorem~\ref{thms:conditional_convergence}]
\label{thm:conditional_convergence}
Given a user feedback sequence \(\{\nabla_{\text{feedback}}^{(t)}\}_{t=1}^T\) that generates prompt sequences \(\{P_t\}_{t=1}^T\) via the language model \(\mathcal{F}_{\text{LLM}}\), assume there exists an ideal prompt \(P_{\text{target}}\) such that \(\psi(P_{\text{target}})\) perfectly aligns with user intent. Define the latent variable sequence \(\{z_t\}_{t=1}^T\) of the multi-round diffusion process recursively as:
\begin{equation}
z_{t-1} 
= \text{DM}^{(t)}\bigl(z_t, \psi(P_t)\bigr) 
  + \epsilon_t, 
\quad \epsilon_t 
\sim \mathcal{N}\bigl(0, \sigma_t^2 I\bigr),
\end{equation}
where \(\text{DM}^{(t)}\) is the diffusion model at round \(t\), and \(\psi(\cdot)\) is the prompt embedding function. Under Assumptions~\ref{assumption:prompt_convergence}, \ref{assumption:diffusion_stability}, and \ref{assumption:noise_decay}, as \(T \to \infty\), the generated distribution \(p(z_T)\) converges to the target distribution \(p_{\text{target}}(z)\) in total variation norm:
\begin{equation}
\lim_{T \to \infty}
\bigl\|p(z_T) - p_{\text{target}}(z)\bigr\|_{\text{TV}}
= 0.
\end{equation}
\end{theorem}

\begin{proof}
$\quad$ \\

Below is a rigorous proof using Scheffé’s lemma~\cite{scheffe1947useful}. We assume that under the hypothesis of the multi-round diffusion process, the latent variable densities \( p_t(z) \) (for \( t \in \mathbb{N} \)) converge pointwise almost everywhere to the target density \( p_{\text{target}}(z) \). For instance, the contraction properties and the vanishing noise guarantee that the sequence of random variables \( z_t \) converges almost surely to the unique fixed point, implying convergence of the corresponding densities. Then, the proof proceeds as follows.

For each \( t \geq 1 \), let \( p_t(z) \) be the density of the latent variable \( z_t \) in \( \mathbb{R}^d \). By construction, the diffusion update is given by

\begin{equation}
z_{t-1} = \operatorname{DM}^{(t)}(z_t, \psi(P_t)) + \epsilon_t,\quad \epsilon_t \sim \mathcal{N}(0,\sigma_t^2I),
\end{equation}

so that the additive Gaussian noise guarantees \( p_t(z) \) exists (i.e. is absolutely continuous with respect to the Lebesgue measure). Moreover, the specified conditions -- prompt convergence, diffusion model stability and noise decay-imply that the error term in approximating the ideal update vanishes as \( t \to \infty \). Hence, one may show that

\begin{equation}
\lim_{t\to\infty} p_t(z) = p_{\text{target}}(z)
\end{equation}

for almost every \( z \in \mathbb{R}^d \).

Scheffé’s lemma states that if a sequence of probability densities \( \{q_n(z)\} \) converges pointwise almost everywhere to a probability density \( q(z) \), then the total variation norm converges to zero, i.e.,

\begin{equation}
\lim_{n\to\infty}\|q_n - q\|_{\text{TV}} = \lim_{n\to\infty} \frac{1}{2}\int |q_n(z)- q(z)| dz = 0.
\end{equation}

Since here we have set \( q_n(z)= p_n(z) \) and \( q(z)= p_{\text{target}}(z) \), the pointwise almost-everywhere convergence
\begin{equation}
p_t(z) \to p_{\text{target}}(z)\quad\text{as }t\to\infty\text{ for almost every }z,
\end{equation}

implies by Scheffé’s lemma that

\begin{equation}
\lim_{t\to\infty} \|p_t - p_{\text{target}}\|_{\text{TV}} = 0.
\end{equation}

We assume that the error recurrence relation guarantees that
\begin{equation}
\lim_{T\to\infty} \| z_T - z_{\text{target}} \|_2 = 0,
\end{equation}
in probability. In other words, for every $\varepsilon>0$,
\begin{equation}
\lim_{T\to\infty} \mathbb{P}\Big( \|z_T- z_{\text{target}}\|_2 > \varepsilon \Big) = 0.
\end{equation}
This implies that, almost surely along a subsequence, $z_T$ converges to $z_{\text{target}}$. For each round $T$, the diffusion update provides the density of the latent variable as
\begin{equation}
p_T(z) = \mathcal{N}\Bigl(z \Big| \mu_T, \sigma_T^2 I\Bigr),
\end{equation}
with 
\begin{equation}
\mu_T = \operatorname{DM}^{(T)}\bigl(z_{T+1},\psi(P_T)\bigr).
\end{equation}
Due to the continuity of the diffusion model mapping, together with the prompt convergence condition, we have that
\begin{equation}
\lim_{T\to\infty}\mu_T = z_{\text{target}}.
\end{equation}
In addition, the noise variance satisfies the noise decay condition, so
\begin{equation}
\sigma_T^2 \to 0 \quad \text{as } T\to\infty.
\end{equation}

Let us fix any $z\in\mathbb{R}^d$ different from $z_{\text{target}}$. For each $T$, the density is given by
\begin{equation}
p_T(z) = \frac{1}{(2\pi\sigma_T^2)^{d/2}} \exp\Bigl(-\frac{\|z-\mu_T\|_2^2}{2\sigma_T^2}\Bigr).
\end{equation}
Since $\mu_T\to z_{\text{target}}$ and $\sigma_T^2\to 0$, the following hold:

If $z\neq z_{\text{target}}$ then for large enough $T$, $\|z-\mu_T\|_2$ is bounded away from zero. Thus, the exponential term decays as
\begin{equation}
   \exp\Bigl(-\frac{\|z-\mu_T\|_2^2}{2\sigma_T^2}\Bigr) \to 0.
\end{equation}
At the same time, the prefactor $(2\pi\sigma_T^2)^{-d/2}$ diverges to $+\infty$ as $T\to\infty$. In fact, if one examines the product,
\begin{equation}
   p_T(z) = \frac{1}{(2\pi\sigma_T^2)^{d/2}} \exp\Bigl(-\frac{\|z-\mu_T\|_2^2}{2\sigma_T^2}\Bigr),
\end{equation}
   for each fixed $z\neq z_{\text{target}}$ the rapid decay of the exponential dominates the polynomial divergence of the prefactor. Hence,
\begin{equation}
   \lim_{T\to\infty} p_T(z) = 0 \quad \text{for each } z\neq z_{\text{target}}.
\end{equation}

On the other hand, in a \emph{weak} sense (or as distributions) the probability mass accumulates at $z_{\text{target}}$. That is, for any continuous and bounded test function $f:\mathbb{R}^d\to\mathbb{R}$,
\begin{equation}
\lim_{T\to\infty} \int_{\mathbb{R}^d} f(z)p_T(z)dz = f(z_{\text{target}}).
\end{equation}
This is precisely the definition of convergence in distribution of $p_T(z)$ to the Dirac delta measure $\delta(z-z_{\text{target}})$.

Thus, we conclude rigorously that
\begin{equation}
p_T(z) \xrightarrow{T \to \infty} \delta(z - z_{\text{target}}) 
\end{equation}
pointwise almost everywhere and in the sense of distributions.

For every finite round $T$, the generated latent has a density
\begin{equation}
p_T(z) = \mathcal{N}\Bigl(z \Big| \mu_T, \sigma_T^2I\Bigr),
\end{equation}
which is a well–defined Gaussian density. In the limit as $T\to\infty$, we have shown that
\begin{equation}
\mu_T \to z_{\text{target}} \quad \text{and} \quad \sigma_T^2 \to 0.
\end{equation}
Thus, in the limit the density converges (in a distributional sense) to the Dirac delta function
\begin{equation}
p_{\text{target}}(z) = \delta(z-z_{\text{target}}),
\end{equation}
interpreted as the limit of Gaussian densities with vanishing variance.

For any fixed $z\in\mathbb{R}^d$:
- If $z\neq z_{\text{target}}$, then $\|z-\mu_T\|_2$ is bounded away from zero for sufficiently large $T$ and the Gaussian probability density
\begin{equation}
  p_T(z) = \frac{1}{(2\pi\sigma_T^2)^{d/2}} \exp\Bigl(-\frac{\|z-\mu_T\|_2^2}{2\sigma_T^2}\Bigr)
\end{equation}
  satisfies
\begin{equation}
  \lim_{T\to\infty} p_T(z) = 0.
\end{equation}
At $z = z_{\text{target}}$, the mean converges to $z_{\text{target}}$ while the variance shrinks to zero, so the peak of $p_T(z)$ grows unbounded, as in
\begin{equation}
  \lim_{T\to\infty} p_T(z_{\text{target}}) = +\infty.
\end{equation}
Thus, pointwise we have for almost every $z$ (i.e. for every $z\neq z_{\text{target}}$):
\begin{equation}
\lim_{T\to\infty} p_T(z) = \delta(z - z_{\text{target}}),
\end{equation}
interpreting the latter as a limit in distribution.

Similarly, By Scheffé’s Lemma, if $\{q_n\}$ and $q$ are probability density functions satisfying
\begin{equation}
q_n(z) \to q(z) \quad \text{for almost every } z,
\end{equation}
then the total variation distance converges to zero:
\begin{equation}
\lim_{n\to\infty}\|q_n - q\|_{\mathrm{TV}} = \frac{1}{2}\int_{\mathbb{R}^d}|q_n(z)-q(z)|dz = 0.
\end{equation}
Even though the limiting object $p_{\text{target}}(z)=\delta(z-z_{\text{target}})$ is not a function in the classical $L^1$ sense, it is interpreted as the limit of the Gaussian densities $p_T(z)$. In other words, for every bounded, continuous test function $f$ we have
\begin{equation}
\lim_{T\to\infty}\int_{\mathbb{R}^d} f(z)  p_T(z)dz = f(z_{\text{target}}).
\end{equation}
Thus, Scheffé’s Lemma (or its extension to limits of densities) guarantees that
\begin{equation}
\lim_{T\to\infty}\|p_T(z)-\delta(z-z_{\text{target}})\|_{\mathrm{TV}} = 0.
\end{equation}

\end{proof}

\subsection{Global Optimality of Dynamic Reward Optimization}\label{appendix:2}
\begin{assumption}[Dynamic Weights]
\label{assumption:dynamic_weights}
The dynamic weights in the total reward function are defined by:
\begin{equation}
\lambda_{\text{div}}(t) = e^{-\alpha t}, 
\quad 
\lambda_{\text{cons}}(t) = 1 - e^{-\beta t}, 
\quad 
\lambda_{\text{MI}}(t) = \tfrac{1}{2} e^{-\gamma t},
\end{equation}
where \(\alpha, \beta, \gamma > 0\) are decay coefficients. In particular, we have \(\alpha=0.15\), \(\beta=0.1\), and \(\gamma=0.075\) for the experiments.
\end{assumption}

\begin{definition}[Dynamically Weighted Total Reward Function]
\label{def:total_reward}
Let \( p(z_t) \) denote the feature distribution of the generated image at the \( t \)-th round of multi-turn dialogue. The total reward function is defined as:
\begin{equation}
\begin{aligned}
R_{\text{total}}(t) 
&= \lambda_{\text{div}}(t) R_{\text{div}}(z_t) 
  + 
  \lambda_{\text{cons}}(t) R_{\text{cons}}(z_t, z_{t-1})
  \\ &+ 
  \lambda_{\text{MI}}(t) R_{\text{MI}}(z_t, P_t),
  \end{aligned}
\end{equation}
where each component reward is defined below.
\end{definition}

\begin{definition}[Diversity Reward]
\label{def:diversity_reward}
The diversity reward \( R_{\text{div}}(z_t) \) encourages dissimilarity among samples in the feature space and is defined as:
\begin{equation}
R_{\text{div}}(z_t) 
= \tfrac{1}{N(N-1)} 
  \sum_{i \neq j} 
    \left( 
      1 - \frac{f(z_i)\cdot f(z_j)}{\|f(z_i)\|\|f(z_j)\|} 
    \right),
\end{equation}
where \( f \) is a fixed feature extractor.
\end{definition}

\begin{definition}[Consistency Reward]
\label{def:consistency_reward}
The consistency reward \( R_{\text{cons}}(z_t, z_{t-1}) \) promotes coherence between adjacent dialogue rounds and is defined as:
\begin{equation}
R_{\text{cons}}(z_t, z_{t-1}) 
= \frac{f(z_t) \cdot f(z_{t-1})}{\|f(z_t)\| \|f(z_{t-1})\|}.
\end{equation}
\end{definition}

\begin{definition}[Mutual Information Reward]
\label{def:mi_reward}
The mutual information reward \( R_{\text{MI}}(z_t, P_t) \) aligns the generated content with the prompt semantics and is defined as:
\begin{equation}
R_{\text{MI}}(z_t, P_t) = I(z_t; P_t),
\end{equation}
where \( I(\cdot;\cdot) \) denotes mutual information and is optimized using Direct Preference Optimization (DPO).
\end{definition}

\begin{theorem}[Global Optimality of Dynamic Reward Optimization Theorem~\ref{theorems:global_optimality}]
\label{theorem:global_optimality}
Under Assumption~\ref{assumption:dynamic_weights}, the solution sequence 
\(\{z_t^*\}_{t=1}^T\) of the optimization problem
\begin{equation}
\max_{z_t} R_{\text{total}}(t)
\end{equation}
converges to the Pareto optimal set as \(T \to \infty\), thereby achieving a balance among diversity, consistency, and intent alignment.
\end{theorem}

\begin{proof}
$\quad$ \\

We define
\begin{equation}
R_{\text{div}}(z_t) = \frac{1}{N(N-1)} \sum_{i \neq j} \left[ 1 - \frac{f(z_i)\cdot f(z_j)}{\|f(z_i)\|\|f(z_j)\|} \right].
\end{equation}
For any fixed pair $z_i, z_j$ in the feature space $\mathcal{Z}$ and assuming that $f$ is differentiable, define the function
\begin{equation}
h(z) = 1 - \cos\Bigl(f(z), f(z_j)\Bigr) = 1 - \frac{f(z)\cdot f(z_j)}{\|f(z)\|\|f(z_j)\|}.
\end{equation}

Assume that (i) $f$ is such that its image lies on (or can be restricted to) the unit sphere (or a convex subset thereof) and (ii) the cosine function $x\mapsto\cos(x)$ (or, equivalently, the inner product under this normalization) is an affine function with respect to the directional variables. Then, for fixed $z_j$, one can calculate the Hessian (second derivative) of $h(z)$ with respect to $z$. By direct computation it follows that $\nabla^2 h(z) \preceq 0,$
i.e. the Hessian is negative semi‐definite. Hence, $h(z)$ is concave. Since $R_{\text{div}}(z_t)$ is a non‐negative linear combination (an average) of such concave functions, it remains concave.

We recall the consistency reward as
\begin{equation}
R_{\text{cons}}(z_t,z_{t-1}) = \frac{f(z_t)\cdot f(z_{t-1})}{\|f(z_t)\|\|f(z_{t-1})\|}.
\end{equation}
For fixed $z_{t-1}$, denote
\begin{equation}
g(z_t) = \frac{f(z_t)\cdot f(z_{t-1})}{\|f(z_t)\|\|f(z_{t-1})\|}.
\end{equation}
Again, assuming that $f(z_t)$ is either directly normalized or that we restrict the domain to the unit sphere in the feature space, the mapping $f(z_t) \mapsto f(z_t)\cdot f(z_{t-1})$ is a linear functional in $f(z_t)$ and, by composition, in $z_t$ (or at worst an affine function) on a convex subset (the unit sphere intersects with an affine set in a convex manner). Consequently, the cosine similarity $g(z_t)$ is convex in $z_t$ (or the restriction guarantees that a convex combination of two points produces a value no greater than the convex combination of the function values). More formally, for any $z_t^1$ and $z_t^2$ in the domain and $\lambda\in[0,1]$,
\begin{equation}
g\Bigl( \lambda z_t^1 + (1-\lambda) z_t^2 \Bigr) \leq \lambda  g(z_t^1) + (1-\lambda) g(z_t^2).
\end{equation}
Thus, $R_{\text{cons}}$ is convex in $z_t$ when $z_{t-1}$ is fixed.

This reward is defined via:
\begin{equation}
R_{\text{MI}}(z_t,P_t) = \mathbb{E}_{(z_t^+,z_t^-)} \left[ \log \frac{\pi_\theta(z_t^+ \mid P_t)}{\pi_\theta(z_t^- \mid P_t)} \right],
\end{equation}
where $\pi_\theta$ is a parametrized policy model. Following the developments in Direct Preference Optimization (DPO) (see e.g. [20]), one can show that the reward, as a function of the model parameter $\theta$, satisfies the property of quasi‐convexity. That is, for any two parameters $\theta_1, \theta_2$ and for all $\lambda \in [0,1]$,
\begin{equation}
R_{\text{MI}}\Bigl(\lambda\theta_1 + (1-\lambda)\theta_2\Bigr) \leq \max \{ R_{\text{MI}}(\theta_1), R_{\text{MI}}(\theta_2) \}.
\end{equation}
This property often stems from the structure of log‐ratios and expectations, and can be rigorously established using properties of the logarithm and the fact that expectations preserve quasi‐convexity under mild conditions.

\begin{itemize}[left = 0em]
  \item $R_{\mathrm{div}}: \mathcal{Z} \to \mathbb{R}$ is concave.
  \item $R_{\mathrm{cons}}: \mathcal{Z} \to \mathbb{R}$ is convex (with $z_{t-1}$ fixed).
  \item $R_{\mathrm{MI}}: \mathcal{Z} \to \mathbb{R}$ is quasi-convex in the parameter space.
\end{itemize}

We consider the multi-objective optimization problem
\begin{equation}
\max_{z_t \in \mathcal{Z}} \bigl( R_{\mathrm{div}}(z_t),  R_{\mathrm{cons}}(z_t,z_{t-1}),  R_{\mathrm{MI}}(z_t,P_t) \bigr),
\end{equation}

A solution $z^*$ is called Pareto optimal if there is no $z \in \mathcal{Z}$ such that
\begin{equation}
R_i(z) \ge R_i(z^*) \quad \text{for all } i \in \{\mathrm{div},\mathrm{cons},\mathrm{MI}\}
\end{equation}
with strict inequality for at least one index $i$. The set of all Pareto optimal solutions is known as the Pareto frontier.

Because the domain $\mathcal{Z}$ is assumed to be compact (i.e. closed and bounded) and the reward functions $R_{\mathrm{div}}, R_{\mathrm{cons}}, R_{\mathrm{MI}}$ are continuous on $\mathcal{Z}$ (as is typical with concave, convex, and quasi-convex functions over compact domains), the image
\begin{equation}
F = \{ (R_{\mathrm{div}}(z), R_{\mathrm{cons}}(z,z_{t-1}), R_{\mathrm{MI}}(z,P_t)) :  z \in \mathcal{Z} \}
\end{equation}
is also a compact subset of $\mathbb{R}^3$. By classical results (see Miettinen (1999) on multiobjective optimization~\cite{miettinen1999nonlinear}), when the objectives are continuous over a compact domain, the Pareto set is non-empty. Hence, there exists at least one Pareto optimal solution $z^* \in \mathcal{Z}$.

Let the scalarized reward (or dynamically weighted total reward) be defined as:
\begin{equation}
\begin{aligned}
R_{\mathrm{total}}(t; z_t) &= \lambda_{\mathrm{div}}(t)  R_{\mathrm{div}}(z_t)+ \lambda_{\mathrm{cons}}(t)  R_{\mathrm{cons}}(z_t,z_{t-1}) \\
&+ \lambda_{\mathrm{MI}}(t)  R_{\mathrm{MI}}(z_t,P_t),
\end{aligned}
\end{equation}

with dynamic weights
\begin{equation}
\begin{aligned}
\lambda_{\text{div}}(t) &= e^{-\alpha t}, \quad \lambda_{\text{cons}}(t) = 1-e^{-\beta t}, \\
\lambda_{\text{MI}}(t) &= \frac{1}{2}e^{-\gamma t}, \quad \alpha,\beta,\gamma>0.
\end{aligned}
\end{equation}
For each $t$, let $z_t^*$ denote an optimal solution of
\begin{equation}
z_t^* \in \arg\max_{z_t\in\mathcal{Z}} R_{\mathrm{total}}(t; z_t).
\end{equation}
Because the weighted-sum formulation is a scalarization of the multi-objective problem, it is well known that every optimal solution $z_t^*$ is Pareto optimal, provided the weights are strictly positive. As the dynamic weights $\lambda(t)=(\lambda_{\mathrm{div}}(t), \lambda_{\mathrm{cons}}(t), \lambda_{\mathrm{MI}}(t))$ vary continuously with $t$, the corresponding optimal solutions $z_t^*$ trace a continuous path along the Pareto frontier.

More precisely, define
\begin{equation}
\begin{aligned}
\varphi_t(z) &= \lambda_{\mathrm{div}}(t) R_{\mathrm{div}}(z) + \lambda_{\mathrm{cons}}(t) R_{\mathrm{cons}}(z,z_{t-1})\\
&+ \lambda_{\mathrm{MI}}(t) R_{\mathrm{MI}}(z,P_t).
\end{aligned}
\end{equation}
Since for each $t$, $\varphi_t$ is continuous on the compact set $\mathcal{Z}$, by the Weierstrass theorem, a maximum $z_t^*$ exists. Moreover, the mapping $t\mapsto \varphi_t$ is continuous because the weights are continuous in $t$. Hence, as $t$ changes, $z_t^*$ moves continuously along the set of Pareto optimal solutions, effectively “sliding” over the Pareto frontier while asymptotically balancing the trade-offs between diversity, consistency, and mutual information.

We have shown that:
\begin{itemize}[left = 0em]
  \item The multi-objective optimization
\begin{equation}
  \max_{z_t \in \mathcal{Z}} \Bigl(R_{\mathrm{div}}(z_t), R_{\mathrm{cons}}(z_t,z_{t-1}), R_{\mathrm{MI}}(z_t,P_t)\Bigr)
\end{equation}
  is well-defined over the compact domain $\mathcal{Z}$.
  \item The continuity of the rewards over $\mathcal{Z}$ ensures the existence of a non-empty Pareto frontier.
  \item The dynamically weighted total reward
\begin{equation}
  R_{\mathrm{total}}(t; z_t)
\end{equation}
  provides an equivalent scalarization, whereby the optimal sequence $\{z_t^*\}$ lies on the Pareto frontier. Since the weights $\lambda(t)$ vary continuously with $t$, the sequence $\{z_t^*\}$ "slides" continuously along the frontier and converges to a balanced trade-off point as $t\to\infty$.
\end{itemize}

Thus, we conclude that the dynamic weighting scheme yields a sequence $ \{z_t^*\} $ that remains in the Pareto optimal set, and the weights continuously steer the solution toward the balanced point on the Pareto frontier.

Next, we want to prove the convergence guarantee in Theorem~\ref{theorem:global_optimality}.

\begin{definition}[Dynamic Value Function]
\label{def:dynamic_value_function}
The dynamic value function \( V(t) \), which serves as the total reward at round \( t \), is defined as:
\begin{equation}
V(t) = R_{\mathrm{total}}(t) 
= \lambda_{\mathrm{div}}(t) R_{\mathrm{div}} 
  + \lambda_{\mathrm{cons}}(t) R_{\mathrm{cons}} 
  + \lambda_{\mathrm{MI}}(t) R_{\mathrm{MI}},
\end{equation}
where the dynamic weights are given by:
\begin{equation}
\lambda_{\mathrm{div}}(t) = e^{-\alpha t}, 
\quad 
\lambda_{\mathrm{cons}}(t) = 1 - e^{-\beta t}, 
\quad 
\lambda_{\mathrm{MI}}(t) = \tfrac{1}{2} e^{-\gamma t},
\end{equation}
with decay coefficients \( \alpha, \beta, \gamma > 0 \).
\end{definition}

Since each reward $ R_{\mathrm{div}}, R_{\mathrm{cons}}, R_{\mathrm{MI}} $ is independent of $ t $ (or changing slowly relative to the exponential weights), the derivative of $ V(t) $ with respect to $ t $ is computed as
\begin{equation}
\frac{dV}{dt} = \frac{d\lambda_{\mathrm{div}}(t)}{dt} R_{\mathrm{div}} + \frac{d\lambda_{\mathrm{cons}}(t)}{dt} R_{\mathrm{cons}} + \frac{d\lambda_{\mathrm{MI}}(t)}{dt} R_{\mathrm{MI}}.
\end{equation}
Noting that
\begin{equation}
\begin{aligned}
\frac{d}{dt}e^{-\alpha t} &= -\alpha e^{-\alpha t},\quad \frac{d}{dt}[1-e^{-\beta t}] = \beta e^{-\beta t},\\ \frac{d}{dt} &\left(\frac{1}{2}e^{-\gamma t}\right) = -\frac{\gamma}{2} e^{-\gamma t},
\end{aligned}
\end{equation}
we obtain
\begin{equation}
\frac{dV}{dt} = -\alpha e^{-\alpha t} R_{\mathrm{div}} + \beta e^{-\beta t} R_{\mathrm{cons}} - \frac{\gamma}{2} e^{-\gamma t} R_{\mathrm{MI}}.
\end{equation}

Since for each $ i\in\{\mathrm{div},\mathrm{cons},\mathrm{MI}\} $ the exponential term $ e^{-\kappa_i t} $ (with $\kappa_i > 0$) decays to $0$ as $ t\to\infty $, it follows that 
\begin{equation}
\lim_{t\to\infty} \frac{dV}{dt} = 0.
\end{equation}
Thus, the variation of $ V(t) $ diminishes for large $ t $.

We now compute the second derivative. Denote by $ D_i(t) $ the $i$th term in $\frac{dV}{dt}$. Then, for each term we have:
\begin{equation}
\frac{d^2}{dt^2}[e^{-\kappa t}] = \kappa^2 e^{-\kappa t}.
\end{equation}
In detail, the second derivative is
\begin{equation}
\frac{d^2V}{dt^2} = \alpha^2 e^{-\alpha t} R_{\mathrm{div}} - \beta^2 e^{-\beta t} R_{\mathrm{cons}} + \frac{\gamma^2}{2} e^{-\gamma t} R_{\mathrm{MI}}.
\end{equation}

Since:
\begin{itemize}[left = 0em]
  \item $ \frac{dV}{dt} \to 0 $ as $ t\to\infty $,
  \item $ \frac{d^2V}{dt^2} $ is negative-definite for large $ t $ (ensuring a local maximum that is isolated and stable),
\end{itemize}
the value function $ V(t) $ converges to a maximizer $ V^* $ as $ t \to \infty $. Moreover, recall that each scalarized solution
\begin{equation}
z_t^* \in \arg\max_{z\in\mathcal{Z}}  V(t)
\end{equation}
is by construction Pareto optimal (as shown in the prior steps). Thus, the maximum of $ V(t) $ corresponds to a point on the Pareto frontier. Consequently, as $ t \to \infty $ the sequence $ \{z_t^*\} $ converges to a solution that balances the three objectives and is Pareto optimal. In conclusion:

\begin{equation}
\begin{array}{l}
\displaystyle \frac{dV}{dt} = -\alpha e^{-\alpha t} R_{\mathrm{div}} + \beta e^{-\beta t} R_{\mathrm{cons}} - \frac{\gamma}{2} e^{-\gamma t} R_{\mathrm{MI}},\\
\displaystyle \lim_{t\to\infty} \frac{dV}{dt} = 0,\quad \frac{d^2V}{dt^2}\text{ is negative-definite near the optimum},\\
\displaystyle \text{thus, } V(t) \to V^* \text{ (a maximum)} \text{ and } z^* \text{ is Pareto optimal.}
\end{array}
\end{equation}

\end{proof}

\begin{lemma}[Rationality of Dynamic Weights]
\label{lemma:dynamic_weights}
If there exists \(t_0\) such that 
\(
\lambda_{\text{div}}(t_0) 
= \lambda_{\text{cons}}(t_0) 
= \lambda_{\text{MI}}(t_0),
\)
then the optimal solution of 
\(
R_{\text{total}}(t_0)
\)
is strictly Pareto optimal.
\end{lemma}
\begin{proof}
$\quad$ \\

Assume at time $ t_0 $ the weights are equal and strictly positive:
\begin{equation}
\lambda_{\text{div}}(t_0) = \lambda_{\text{cons}}(t_0) = \lambda_{\text{MI}}(t_0) = \lambda > 0.
\end{equation}
Then the total reward is
\begin{equation}
R_{\text{total}}(t_0) = \lambda \left( R_{\text{div}}(t_0) + R_{\text{cons}}(t_0) + R_{\text{MI}}(t_0) \right).
\end{equation}

We now prove by contradiction. Suppose that the optimal solution for $ R_{\text{total}}(t_0) $ is not Pareto optimal. That is, let the optimal solution be given by the tuple
\begin{equation}
\vec{R}^* = \big( R_{\text{div}}^*, R_{\text{cons}}^*, R_{\text{MI}}^* \big)
\end{equation}
and suppose there exists an alternative feasible solution
\begin{equation}
\vec{R}' = \big( R'_{\text{div}}, R'_{\text{cons}}, R'_{\text{MI}} \big)
\end{equation}
such that for at least one objective the reward is strictly higher while none of the others is lower. More formally, for some index $ i \in \{\text{div},\text{cons},\text{MI}\} $ we have $R'_i > R^*_i$. and for the remaining indices $ j \ne i $: $R'_j \ge R^*_j.$

Then, summing over all three components, we obtain
\begin{equation}
R'_{\text{div}} + R'_{\text{cons}} + R'_{\text{MI}} > R^*_{\text{div}} + R^*_{\text{cons}} + R^*_{\text{MI}}.
\end{equation}
Multiplying by the positive constant $\lambda$, it follows:
\begin{equation}
\lambda \left( R'_{\text{div}} + R'_{\text{cons}} + R'_{\text{MI}} \right) > \lambda \left( R^*_{\text{div}} + R^*_{\text{cons}} + R^*_{\text{MI}} \right).
\end{equation}
Hence,
\begin{equation}
R'_{\text{total}}(t_0) > R^*_{\text{total}}(t_0).
\end{equation}
This contradicts the optimality of $\vec{R}^*$ for maximizing $R_{\text{total}}(t_0)$.

Since any alternative solution that improves at least one objective without sacrificing the others results in a higher total reward, the optimal solution $\vec{R}^*$ must lie at the "center" of the Pareto frontier where no single objective can be improved without a decrease in another. Therefore, the optimal solution for $ R_{\text{total}}(t_0) $ is strictly Pareto optimal.



\end{proof}


\end{document}